\documentclass{article}
\usepackage[preprint]{corl_2023}

\usepackage{macros}
\usepackage[ruled,vlined,linesnumbered]{algorithm2e}
\usepackage[caption=false,font=small]{subfig}
\usepackage{enumitem}
\usepackage{footmisc}
\usepackage{xcolor}
\usepackage{booktabs}
\usepackage{multirow}
\usepackage{tikz}
\usetikzlibrary{arrows,backgrounds,decorations,decorations.pathmorphing,positioning,fit,automata,shapes,patterns,plotmarks,calc,shapes.geometric}
\usetikzlibrary{
  arrows.meta,                        
  ext.paths.ortho,                    
  ext.positioning-plus,               
  ext.node-families.shapes.geometric, 
}     
\graphicspath{{figs/}}

\title{Signal Temporal Logic-Guided \\ Apprenticeship Learning}

\author{
  Aniruddh G.~Puranic, Jyotirmoy V.~Deshmukh, Stefanos Nikolaidis\\
  Department of Computer Science\\
  University of Southern California, USA\\
  \texttt{\{puranic, jdeshmuk, nikolaid\}@usc.edu} \\
}

\begin{document}
\maketitle


\begin{abstract}
    Apprenticeship learning crucially depends on effectively learning rewards, and hence control policies from user demonstrations. Of particular difficulty is the setting where the desired task consists of a number of sub-goals with temporal dependencies. The quality of inferred rewards and hence policies are typically limited by the quality of demonstrations, and poor inference of these can lead to undesirable outcomes. In this paper, we show how temporal logic specifications that describe high level task objectives, are encoded in a graph to define a temporal-based metric that reasons about behaviors of demonstrators and the learner agent to improve the quality of inferred rewards and policies. Through experiments on a diverse set of robot manipulator simulations, we show how our framework overcomes the drawbacks of prior literature by drastically improving the number of demonstrations required to learn a control policy.
\end{abstract}

\keywords{Learning from Demonstration, Reinforcement Learning, Formal Methods, Temporal Logic}


\section{Introduction}
\label{sec:intro}
Recent advances in robotics have led to the development of algorithms that extract control policies for autonomous agents from human demonstrations via the paradigm of learning-from-demonstrations (LfD). An interesting sub-area of LfD is the use of demonstrations alongside reinforcement learning (RL) to either (i) initialize policies for the RL agent \cite{abbeel_bcinit_rl} via behavior cloning (BC) \cite{bco_stone} or (ii) infer rewards using inverse RL (IRL) \cite{ng_russell_irl} for tasks from which policies can be extracted - apprenticeship learning via IRL \cite{abbeel_ng_apprenticeship}. However, designing rewards for Markov Decision Processes (MDPs) \cite{sutton_reinforcement_2018} is non-trivial and typically requires expert knowledge in designing reward functions that can ensure safety and efficiency in the extracted RL policies. More importantly, for robots to be robust to perturbations in the environment, it is crucial to capture the overall goals/intentions of demonstrators, i.e., via IRL, rather than merely mimicking them \cite{abbeel_ng_apprenticeship}. Our work draws inspiration from Apprenticeship Learning (AL) \cite{abbeel_ng_apprenticeship} to learn both rewards and policies. 

A drawback of AL is that it relies on demonstrations being optimal, which is seldom the case in real-world scenarios. More recent IRL and BC-based methods that learn from suboptimal demonstrations \cite{ziebart_maximumentropy,mce_irl_ziebart,score_irl,chen_ssrr_2020,brown_DREX} measure optimality or performance based on statistical noise deviation from the true/optimal demonstrations. However, such noise-to-performance measures are extracted {\em empirically} and hence lack formal reasoning that can explain the quality of behaviors. Furthermore, as the core reward-inference algorithm in AL uses IRL, the rewards are inherently Markovian, and they do not account for temporal dependencies among subgoals in demonstrations. Research in reward design \cite{abel2021expressivity, pitis2022rational} discusses the need for non-Markovian reward representations, especially in time-dependent multi-goal RL settings. Such non-Markovian rewards are typically designed using spilt-MDPs \cite{abel2022express} and reward machines \cite{ltl_reward_machines, generic_rewardmachines}, which require significantly increasing the state and/or action spaces of the MDPs thereby increasing the space and computational complexity for the underlying RL algorithms.

To address these limitations, our prior work \cite{puranic_det_corl2020, puranic_stoch_cont_ral2021} has proposed to use Signal Temporal Logic (STL) to define high-level tasks, and evaluate and rank demonstrations to infer rewards. The semantics of STL measure the quality/fitness, which is the degree of task satisfaction by demonstrations. This facilitates holistic temporal-based ranking of demonstrations and agent behaviors to formulate non-Markovian rewards. Our LfD-STL framework can learn from only a handful of even imperfect/suboptimal demonstrations, {\em without the need to augment the MDP spaces}. It has shown to significantly outperform state-of-the-art IRL methods in terms of reward quality, number of demonstrations required and safety of the learned policy. It can also be applied to stochastic and continuous spaces to extract rewards and behaviors consistent with the task specifications. Our recent work proposed PeGLearn \cite{puranic_peglearn_ral2023} to automatically infer non-Markovian rewards for tasks comprising multiple STL objectives, addressing the representation issues discussed in \cite{pitis2022rational}. PeGLearn uses directed graphs to create a partial ordering of specifications to produce a single graph - {\em performance graph} - that holistically captures the demonstrated behaviors.

While the LfD-STL framework with PeGLearn can offer assurances in safety of the learned rewards and policy, it does not explicitly reason about the performance of the learned RL policy. The reason being that LfD-STL is an open-loop framework where the inferred rewards are fixed and are not guaranteed to be optimal without any exploration. Without feedback from agent exploration, it may be impossible to discover better behaviors. We aim to address this issue by using the performance graph as a metric, which we refer to as the {\em performance-graph advantage (PGA)} to guide the RL process. We propose the AL-STL framework that extends LfD-STL with closed-loop learning wherein the reward function and policy are updated iteratively. PGA can be interpreted as the quantification of the areas for improvement of the policy, and is optimized alongside appropriate existing RL algorithms. This enables reasoning about possibly new behaviors that were not demonstrated before, but still satisfy the task specifications. The key insight of our work is that {\em a cumulative/collective measure of (multiple) task objectives along with exploration in the neighborhood of observed behaviors guides the refinement of rewards and policies that can extrapolate beyond demonstrated behaviors}. Our contributions are summarized as follows:
\begin{enumerate}
\item We propose AL-STL, a novel extension to the LfD-STL framework to enable closed-loop learning of the reward function and policy.
\item We quantify STL-based performance graphs learned via PeGLearn in terms of an advantage function to guide the RL training process, and formally reason about policy improvements when demonstrations are suboptimal. 
\item We evaluate our approach on a variety of robotic manipulation tasks and discuss how our framework outperforms state-of-the-art literature. 
\end{enumerate}

\section{Related Works}
\label{sec:rel_work}

Learning-from-demonstrations (LfD) to extract control policies can be broadly classified into two main categories based on the underlying intentions: (i) imitation learning (IL), such as behavior cloning (BC) via supervised learning \cite{bco_stone}, where the objective is to directly mimic the actions of the demonstrators, and (ii) inverse reinforcement learning (IRL)  \cite{ng_russell_irl, ziebart_maximumentropy, abbeel_ng_apprenticeship}, where the objective is to characterize the overall goal of the demonstrators via cost/reward functions. 

Learning rewards via entropy-enabled IRL \cite{ziebart_maximumentropy, mce_irl_ziebart, adv_irl, gaifo} regard suboptimal demonstrations as noisy deviations from the optimal statistical model, and hence require access to many demonstrations. Learning better policies from suboptimal demonstrations has been explored in \cite{brown_DREX}. This method injects noise into trajectories to infer a ranking, however it synthetically generates trajectories via BC which has issues with covariate shift and induces undesirable bias. \cite{chen_ssrr_2020} addresses this by defining a relation between injected noise and performance. However, this noise-performance relationship is empirically derived and lacks formal reasoning. Score-based IRL \cite{score_irl} uses expert-scored trajectories to learn a reward function, relying on a large set of nearly-optimal demonstrations and hence generating scores for each of them. Additionally, rewards learned via IRL-based methods are Markovian by nature and typically suited to single-goal tasks, as discussed in prior work \cite{puranic_det_corl2020,puranic_stoch_cont_ral2021}.

In the area of LfD with temporal logics, the closest to our work is a counterexample-guided approach using probabilistic computation tree logics (PCTL) for safety-aware AL \cite{wenchao_safety_al}. Our work differs from it in two significant ways: (i) we use STL which is applicable to continuous spaces and offers timed-interval semantics, which are lacking in PCTL, and (ii) the reward inference algorithm in \cite{wenchao_safety_al} relies on IRL, while ours is based on LfD-STL \cite{puranic_peglearn_ral2023}, which greatly improves sample complexity, accuracy and inference speed. Trade-offs for multi-objective RL have been explored in \cite{cho_mpc_stl} by explicitly defining specification priorities beforehand. Alternate approaches convert specifications to their equivalent automaton and augment it to the MDP states \cite{belta_lfd,memarian_cdc2020,wen_lfd_highlevel}. In our work, we do not alter the MDP structure, thereby avoiding the drawbacks of increased space and computational complexities of augmented MDPs.

\section{Preliminaries}
\label{sec:prelims}

\subsection{Mathematical Notations}
\label{sec:math}

The interactions between the agent (robot) and the environment are modeled with a Markov Decision Process.
\begin{definition}[Markov Decision Process (MDP)]
    An MDP is given by a tuple $M=\MDP$ where $\StateSpace \subset \Reals^k$ is the state space and $\ActionSpace \subset \Reals^l$ is the action space of the system; $T$ is the transition function, where $T(s, a, s') = Pr(s'\mid s,a)$; $R$ is a reward function that typically maps either some $s \in \StateSpace$, state-action pair $\StateSpace \times \ActionSpace$ or some transition $\StateSpace \times \ActionSpace \times \StateSpace$ to $\Reals$.
\end{definition}
The goal of RL is to find a policy $\pi : \StateSpace \times \ActionSpace \rightarrow [0,1]$ that maximizes the total (discounted) reward from performing actions on an MDP, i.e., the objective is to compute $\max \sum_{t = 0}^{\infty}{\gamma^t r_t}$, where $r_t$ is the output of the reward function $R$ at $t$ and $\gamma$ is the discount factor. In this paper, we assume full observation of the state space for MDPs. 

\begin{definition}[Trajectory or Episode Rollout]
	A trajectory in an MDP is a sequence of state-action pairs of finite length $L \in \mathbb{N}$ by following some policy $\pi$ from an initial state $s_0$, i.e., a trajectory $\tau =\langle s_0,a_0,\cdots,s_L \rangle$, where $s_i \in \StateSpace$ and $a_i \in \ActionSpace$. 
    \label{def:rollout}
\end{definition}

In our LfD setting, the demonstrations are collected on the robot itself (e.g., via teleoperation or kinesthetic teaching), so the observations are elements of the MDP state and action spaces. Hence, we interchangeably refer to trajectories or rollouts as demonstrations. For intuition, we use demonstrations to refer to rollouts provided to the RL agent as inputs, and represent by $\demo$.

Prior work in LfD \cite{puranic_det_corl2020,puranic_stoch_cont_ral2021,puranic_peglearn_ral2023} uses Signal Temporal Logic (STL) \cite{stl_complexity, donze_robust_2010} to define high-level tasks.

\paragraph*{Signal Temporal Logic (STL)} STL is a real-time logic, generally interpreted over a dense-time domain for signals whose values are from a continuous metric space (such as $\Reals^n$). The basic primitive in STL is a {\em signal predicate} $\mu$ that is a formula of the form $f(\vx(t)) > 0$, where $\vx(t)$ is the tuple $(s,a)$ of the trajectory $\vx$ at time $t$, and $f$ maps the signal domain $\domain = (\StateSpace \times \ActionSpace)$ to $\Reals$. STL formulas are then defined recursively using Boolean combinations of sub-formulas, or by applying an interval-restricted temporal operator to a sub-formula.  The syntax of STL is formally defined as follows: $\varphi ::=  \mu \mid \neg \varphi \mid \varphi \wedge \varphi \mid \alw_{I} \varphi  \mid \ev_{I} \varphi \mid \varphi \until_{I} \varphi$. Here, $I = [a,b]$ denotes an arbitrary time-interval, where $a,b\in\Reals^{\ge
0}$. The semantics of STL are defined over a discrete-time signal $\sig$ defined over some time-domain $\timedomain$. The Boolean satisfaction of a signal predicate is simply \textit{True} ($\top$) if the predicate is satisfied and \textit{False} ($\bot$) if it is not, the semantics for the propositional logic operators $\neg, \land$ (and thus $\lor, \rightarrow$) follow the obvious semantics. The following behaviors are represented by the temporal operators:
\begin{itemize}
    \item At any time $t$, $\always_I(\varphi)$ says that $\varphi$ must hold for all samples in $t+I$.
    \item At any time $t$, $\eventually_I(\varphi)$ says that $\varphi$ must hold \textit{at
    least once} for samples in $t+I$.
    \item At any time $t$, $\varphi \until_I \psi$ says that $\psi$ must hold at some time $t'$ in $t+I$, and in $[t,t')$, $\varphi$ must hold at all times.
\end{itemize}

The quantitative (robustness) semantics of STL, defined in \cite{fainekos_robustness_2009,donze_robust_2010}, capture the performance of trajectories. Directed acyclic graphs are used to encode the preferences or performance of the demonstrators. Such graphs provide a convenient way to interpret reward functions for RL tasks.

\begin{definition}[Directed Acyclic Graph (DAG)]
A directed graph is an ordered pair $G = (V, E)$ where $V$ is a set of elements called nodes and $E$ is a set of ordered pairs of nodes called edges, which are directed from one node to another. An edge $e = (u, v)$ is directed from node $u$ to node $v$. A \textit{DAG} is a directed graph that has no directed cycles, i.e., it can be topologically ordered. 
\end{definition}

A path $x \leadsto y$ in $G$ is a set of nodes starting from $x$ and ending at $y$ by following the directed edges from $x$. The ancestors of a node $v$ is the set of all nodes in $G$ that have a path to $v$. Formally, $ancestor(v) = \{u \mid u \leadsto v, u \in V\}$. In our setting, we use a weighted DAG, where each node $v \in V$ is associated with a pair of real numbers - value and weight of the node, represented by $\nu(v)$ and $w(v)$ respectively. Each edge $(u,v) \in E$ is associated with a real number - weight of the edge, represented by $w(u,v)$. Note the difference in number of arguments in the notations for edge and node weights.

\subsection{Reward Inference from Demonstrations and Specifications}
\label{sec:lfdstl_peglearn}

In LfD-STL, the reward function $R$ of the MDP is unknown, instead, it is presented with a finite set of high-level task descriptions in STL $\Phi = \{\phix{1}, \phix{2}, \cdots, \phix{n}\}$ and a finite set of demonstrations $\demoset = \{\demo_1, \demo_2, \cdots, \demo_m\}$, from which the reward function and policy must be inferred. 

\paragraph*{LfD-STL Framework} For a specification $\varphi \in \Phi$ and a demonstration $\demo \in \demoset$ defined as in \autoref{def:rollout}, the value $\rho(\varphi, \demo, t)$ represents how well the demonstration satisfied the given specification from time $t$, which is the quality of the demonstration. To evaluate the entire trajectory, the robustness is defined at $t=0$, i.e. $\rho(\varphi, \demo, 0)$ and is implicitly denoted by $\rho(\varphi, \demo)$. For a demonstration $\demo$, we have an array of evaluations over $\Phi$, given by $\rhod = [\rho(\phix{1},\demo), \cdots, \rho(\phix{n},\demo)]^T$.

Then, for each $\demo \in \demoset$, a local DAG $G_\demo$ is initially constructed via the PeGLearn algorithm \cite{puranic_peglearn_ral2023}, wherein, (i) each task specification $\varphi \in \Phi$ is represented by a node, with the value of the node indicating the fitness of $\demo$ for $\varphi$, i.e., $\nu(\varphi) = \rho(\varphi, \demo)$, and (ii) the edges, along with their corresponding weights encode information about the preferences or performance between every pair of specifications as exhibited by the behavior. For any edge $e(\phix{i},\phix{j})$, its weight, defined by $\nu(\phix{i})-\nu(\phix{j})$, indicates a measure by which the value of $\phix{j}$ must be increased to match the value of $\phix{i}$. As an edge in $G_\demo$ is always directed from a higher-valued node to a lower-valued node, the edge weight is always positive. Absence of an edge between a pair of nodes indicates a zero-weighted edge. {\em Note that this local DAG is applicable to all trajectories that conform to \autoref{def:rollout}}. Thus, PeGLearn maps a trajectory $\tau$ and $\Phi$ to a DAG $G_\tau$. In our work, since specifications can be of different scales (e.g., a specification that monitors acceleration, while another monitors distance), we assume that the robustness bounds are known apriori and we normalize/scale the robustness values to be bounded to some $[-\Delta, \Delta]$. Scaling of robustness can be achieved with piece-wise linear functions or smooth semantics \cite{smooth_stl_semantics}. In addition to extracting a DAG for each trajectory, PeGLearn also captures the holistic behavior of a set of trajectories by aggregating their corresponding local DAGs into a global DAG $\globalG$. The nodes in $\globalG$ are weighted to capture the relative pair-wise priorities of specifications based on the node ancestors or dependencies via $w(\varphi) = |\Phi| - ancestor(\varphi)$, illustrated with an example in \autoref{fig:dag_eg}.

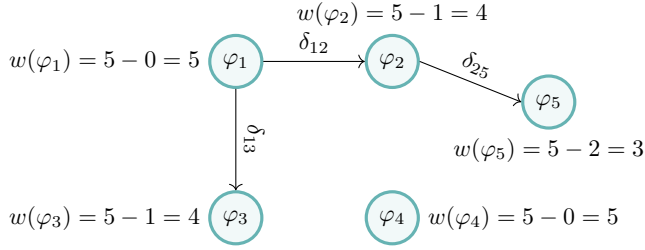
\begin{figure}[htbp]
    \centering
    \tikzset{mypic/.style={scale=1.0, transform shape, node distance={15mm}},
	roundnode/.style={circle, draw=teal!60, fill=teal!5, very thick, minimum size=5mm},
	squarednode/.style={rectangle, draw=red!60, fill=red!5, very thick, minimum size=5mm},
}

\begin{tikzpicture}[mypic,scale=0.9]
    \node[roundnode, label=left:{$w(\phix{1})=5-0=5$}](phi1){$\varphi_1$};
    \node[roundnode, label=above:{$w(\phix{2})=5-1=4$}](phi2)[right=of phi1]{$\varphi_2$};
    \node[roundnode](phi3)[below=of phi1, label=left:{$w(\phix{3})=5-1=4$}]{$\varphi_3$};
    \node[roundnode](phi4)[below=of phi2, label=right:{$w(\phix{4})=5-0=5$}]{$\varphi_4$};
    \node[roundnode](phi5)[right=of phi2, yshift=-0.6cm, label=below:{$w(\phix{5})=5-2=3$}]{$\varphi_5$};
    
    \draw[->] (phi1.east) -- node[midway, above, sloped] {$\delta_{12}$} (phi2.west);
    \draw[->] (phi1.south) -- node[midway, above, sloped] {$\delta_{13}$} (phi3.north);
	\draw[->] (phi2.east) -- node[midway, above, sloped] {$\delta_{25}$} (phi5.west);
\end{tikzpicture}
    \caption{Weights on nodes (specifications) in a DAG.}
    \label{fig:dag_eg}
\end{figure}

The node weights are used to induce bias towards specifications during inference of the reward function and hence the RL policy. Prior literature in behavior modeling with reward functions \cite{mce_irl_ziebart,ziebart_maximumentropy,chen_ssrr_2020} has shown that the performance variations in trajectories obey an exponential form. So, the weights of the specifications from the DAG are normalized with \texttt{softmax} to ensure $\sum_{i=1}^{n} w(\phix{i}) = 1$. We now have a weight vector $w_\Phi = [w(\phix{1}), w(\phix{2}), \cdots, w(\phix{n})]^T$. Each demonstration is then assigned a cumulative robustness/fitness value based on these weights, given by $r_\demo = \rhod^T \cdot w_\Phi$. To generalize, all trajectories are associated with a corresponding performance DAG and cumulative fitness. Once the cumulative fitness is assigned to each demonstration, the demonstrations are ranked based on their $r_\demo$ and the rank-scaled rewards are propagated to the observations via the reward inference method described in \cite{puranic_stoch_cont_ral2021}. In short, the method assigns monotonically increasing rewards (i.e., partial cumulative fitness) to the observed states and/or actions in demonstrations that satisfy the specification, while negative rewards are assigned to states in demonstrations that violate the task specifications.

\section{Methodology}
\label{sec:method}

\subsection{Problem Formulation}
\label{sec:problem}

For an MDP\textbackslash R, we are given: (i) a finite dataset of demonstrations $\demoset = \{\demo_1, \demo_2, ..., \demo_m\}$ and (ii) a set of specifications $\Phi = \{\phix{1}, \phix{2}, ..., \phix{n}\}$ unambiguously expressing the tasks to be performed. The objective is to infer rewards and extract a behavior or control policy for an agent such that its behavior is at least as good or better than the demonstrations, and maximizes the satisfaction of the task specifications. The satisfaction of task specifications is conveyed through the learned reward function that the RL agent seeks to maximize.

More formally, consider a policy $\pi$ under the reward function $R$ that captures the degree of satisfaction of $\Phi$. Let $\tau$ indicate a trajectory obtained by a rollout of $\pi$ in an RL episode. Then, our objective is to find

\begin{equation*}
	\pi^*,R^* = \operatorname*{argmax}_{\pi, R} \Expect_{\tau \sim \pi}\sum_{i=1}^{n}\rho(\phix{i}, \tau)
\end{equation*}
Since every trajectory $\tau$ is characterized by its associated performance DAG $G_\tau$, where the value of a node indicates the robustness for the specification it represents (\autoref{sec:lfdstl_peglearn}), the summation term is the sum of all nodes. We thus define $\nodesum_\tau \doteq \sum_{i=1}^{n}\nu(\phix{i}) = \sum_{i=1}^{n}\rho(\phix{i}, \tau)$. Then the objective is: \[\pi^*,R^* = \operatorname*{argmax}_{\pi, R} \Expect_{\tau \sim \pi}[\nodesum_\tau]\]

An issue with this formulation occurs when there are multiple task specifications, i.e., $n>1$. This results in multi-objective learning, which can introduce conflicting specifications and hence requires optimal trade-offs. For example, in autonomous driving or in robot manipulation, consider the task of reaching a goal location as quickly as possible while avoiding obstacles. Depending on the obstacle locations, performing highly safe behaviors (i.e., staying as far away from obstacles as possible) might affect the time to reach the goal. Similarly, a behavior that aims to reach the goal in the least time will likely need to compromise on its safety robustness. We thus need to find the behaviors that not only maximize the total robustness, but are also maximally robust to each task specification. We illustrate this with \autoref{eg:problem_eg1}.
\begin{example}
	Consider a task with three specifications $\Phi=\{\phix{1}, \phix{2}, \phix{3}\}$, and consider two trajectories $\tau_1$ and $\tau_2$ with robustness vectors $[3, 0, 1]$ and $[2, 1, 1]$, respectively. The reward function inferred with $\tau_1$ will have the weight for $\phix{1}$ dominate $\phix{2}$ due to the exponential (\texttt{softmax}) component, while the reward function for $\tau_2$ will have more uniform weights over all specifications, albeit with a little bias towards $\phix{1}$ versus others. Thus, while both have the same $\nodesum$, $\tau_2$ is overall more robust w.r.t. all the task specifications due to better trade-offs.
	\label{eg:problem_eg1}
\end{example}

By this reasoning, it is more desirable to not only maximize the overall sum, but also maximally satisfy the individual specifications with trade-offs. So, how do we ensure that optimal trade-offs are achieved while maximizing the main objective? By observation, it is straight-forward to deduce that the sum of {\em absolute} pair-wise differences in robustness of specifications must be minimized. This sum is indeed exactly encoded by the edges of our trajectory DAG (performance graph) formulation, which is a unique characteristic. Recall that the edges between two nodes (specifications) indicate the difference in their robustness values (performance). We thus capture the optimal trade-offs for a trajectory $\tau$ with the sum of all edges in its corresponding DAG $G_{\tau}$, which is given by $\edgesum_\tau = \sum_{e \in G_\tau}e$; each edge is defined in \autoref{sec:lfdstl_peglearn}. Both $\nodesum$ and $\edgesum$ can be computed in {\em linear time} using the same DAG, without additional computational overhead. One might wonder if merely minimizing $\edgesum$ is sufficient for finding the optimal trade-offs. We provide a counterargument in \autoref{eg:problem_eg2}.
\begin{example}
	Consider the same task from \autoref{eg:problem_eg1}, but with two different trajectories $\tau_3$ and $\tau_4$ with robustness vectors $[1, 1, 1]$ and $[-1, -1, -1]$, respectively. Since all the specifications are equally weighted, the $\edgesum$ for both trajectories are the same ($=0$). But clearly, $\tau_3$ is more robust than $\tau_4$ due to the higher $\nodesum$. Furthermore, consider another trajectory $\tau_5$ with vectors $[-1,2, -1]$, whose $\edgesum$ is 6 (i.e., an edge weight is the pair-wise difference when sorted). Between $\tau_4$ and $\tau_5$, the RL agent will prefer $\tau_4$ due to the lower $\edgesum$, which is undesirable.
	\label{eg:problem_eg2}
\end{example}

From both examples, we conclude that the objective is to maximize $\nodesum$ while minimizing $\edgesum$. Our new formulation is, 
\begin{align*}
	\pi^*,R^* &= \operatorname*{argmax}_{\pi, R} \Expect_{\tau \sim \pi}(\nodesum_\tau - \edgesum_\tau)
\end{align*}

As both $\nodesum$ and $\edgesum$ are dependent on each other, this optimization trade-off can be written as:
\begin{equation}
	\pi^*,R^* = \operatorname*{argmax}_{\pi, R} \Expect_{\tau \sim \pi}(\nodesum_\tau - \lambda \cdot \edgesum_\tau)
	\label{eq:graph_adv}
\end{equation}

The constant $\lambda \in [0, 1)$ acts as a regularizer to penalize behaviors with dominant specifications as in \autoref{eg:problem_eg1}, and is a tunable hyperparameter based on the environment. The formulation is very intuitive because we want to extract the optimal DAG which has no edges. Recall that edges are added only if there is a difference between the node values (i.e., robustness). Ideally, if the policy is optimal, then every rollout has the same maximum robustness for all $\Phi$ and so no edges are created. This representation offers the unique ability of providing an intuitive graphical representation of behaviors for interpretability \cite{puranic_peglearn_ral2023}, and formulating an optimization problem. As the robustness of each specification is bounded in $[-\Delta, \Delta]$, the $\nodesum$ for any trajectory is bounded to $[-n\Delta,n\Delta]$. At either limit, the $\edgesum$ is 0, indicating that all nodes in the resulting global DAG $\globalG$ have equal weights ($=1/n$) at the extrema. We will refer to the term $(\nodesum_\tau - \lambda \cdot \edgesum_\tau)$ as {\em performance graph advantage (PGA)}. Analogous to the advantage function in RL, PGA provides information about the scope for improvement (extra possible rewards) under the current reward function and policy. Depending on the type of RL algorithm (on-policy or off-policy) used for training, PGA can be used along with the episode returns as a bonus term, or in the loss function by minimizing the negative of PGA.

\subsection{Framework and Algorithm}
\label{sec:framework}
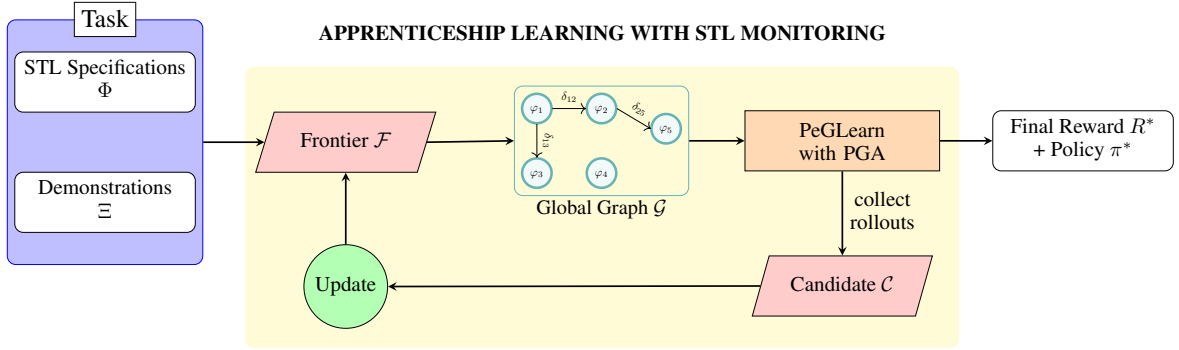
\begin{figure*}[htbp]
	\centering
	\tikzstyle{io} = [rectangle, rounded corners, 
minimum width=3cm, 
minimum height=1cm,
text centered, 
draw=black, 
fill=white!30]

\tikzstyle{buffers} = [trapezium, 
trapezium stretches=true, 
trapezium left angle=70, 
trapezium right angle=110, 
minimum width=3cm, 
minimum height=1cm, text centered, 
draw=black, fill=red!20]

\tikzstyle{process} = [rectangle,
minimum width=3cm, 
minimum height=1cm, 
text centered, 
text width=3cm, 
draw=black, 
fill=orange!30]

\tikzstyle{decision} = [circle, 
minimum width=1cm, 
minimum height=1cm, 
text centered, 
draw=black, 
fill=green!30]
\tikzstyle{arrow} = [thick,->,>=stealth]

\tikzset{
    ioBlock/.style={shape=rectangle, rounded corners, align=center, minimum width=5mm,fill=blue!20,text=white},
    trainingBlock/.style={shape=rectangle, rounded corners, align=center, minimum width=1cm,minimum height=3cm,fill=yellow!20},
    arr/.style={line width=2.75pt,-Stealth},
    dagBlock/.style={shape=rectangle, rounded corners, align=center, minimum width=1cm,minimum height=3cm,fill=red!20},
    arr/.style={line width=2.75pt,-Stealth}
    }

\tikzset{
  basic box/.style={
    shape=rectangle, rounded corners, align=center, draw=#1, fill=#1!25},
  header node/.style={
    node family/width=header nodes,
    font=\strut\large\rmfamily,
    text depth=+.3ex, fill=white, draw},
  header/.style={%
    inner ysep=+1.5em,
    append after command={
      \pgfextra{\let\TikZlastnode\tikzlastnode}
      node [header node] (header-\TikZlastnode) at (\TikZlastnode.north) {#1}
      node [span=(\TikZlastnode)(header-\TikZlastnode)]
           at (fit bounding box) (h-\TikZlastnode) {}
    }
  },
  fat blue line/.style={ultra thick, blue}
}

\tikzset{mypic/.style={scale=0.7, transform shape, node distance={10mm}},
	roundnode/.style={circle, draw=teal!60, fill=teal!5, very thick, minimum size=5mm},
  outerbox/.style={shape=rectangle, rounded corners, draw=teal!60},
	squarednode/.style={rectangle, draw=red!60, fill=red!5, very thick, minimum size=5mm},
}

\scalebox{0.8}{
\begin{tikzpicture}[scale=0.8, node distance=2cm, font=\strut\rmfamily]


\node (specs) [align=center, io] {STL Specifications \\ $\Phi$};
\node (demos) [align=center, io, below of=specs] {Demonstrations \\ $\demoset$};

\begin{scope}[on background layer]
    \node[fit = (specs)(demos), basic box = blue,
      header = Task] (task) {};
  \end{scope}

\node (frontier) [buffers, right of=task, xshift=2cm] {Frontier $\frontier$};

  \begin{scope}[mypic,local bounding box=dagsGroup]

    \node[roundnode](phi1)[right=of frontier, xshift=1.8cm, yshift=10mm]{$\varphi_1$};
    \node[roundnode](phi2)[right=of phi1]{$\varphi_2$};
    \node[roundnode](phi3)[below=of phi1]{$\varphi_3$};
    \node[roundnode](phi4)[below=of phi2]{$\varphi_4$};
    \node[roundnode](phi5)[right=of phi2, yshift=-0.6cm]{$\varphi_5$};

    \draw[->] (phi1.east) -- node[midway, above, sloped] {$\delta_{12}$} (phi2.west);
    \draw[->] (phi1.south) -- node[midway, above, sloped] {$\delta_{13}$} (phi3.north);
    \draw[->] (phi2.east) -- node[midway, above, sloped] {$\delta_{25}$} (phi5.west);

    \node (dags) [outerbox,fit=(phi1)(phi2)(phi3)(phi4)(phi5), label=below:\LARGE{Global Graph $\globalG$}] {};

  \end{scope}

\node (peglearn) [process, align=center, right of=dags, xshift=2cm] {PeGLearn \\ with PGA};
\node (update) [decision, align=center, below of=frontier, yshift=-0.4cm] {Update};
\node (candidate) [buffers, right=of update, below=of peglearn, yshift=0.6cm] {Candidate $\candidate$};
\begin{scope}[on background layer, inner sep=2ex]
    \node (training) [trainingBlock,fit=(frontier)(peglearn)(update)(dags)(candidate),label=above:\textbf{APPRENTICESHIP LEARNING WITH STL MONITORING}] {};
\end{scope}

\node (output) [io, align=center, right of=peglearn, xshift=2cm] {Final Reward $R^*$ \\ + Policy $\pi^*$};


\draw [arrow] (task) -- (frontier.west);
\draw [arrow] (frontier.east) -- (dags.west);
\draw [arrow] (dags.east) -- (peglearn.west);
\draw [arrow] (peglearn.east) -- (output.west);
\draw [arrow] (peglearn.south) -- node[anchor=west, align=center] {collect \\ rollouts} (candidate.north);
\draw [arrow] (candidate.west) -- (update.east);
\draw [arrow] (update.north) -| (frontier.south);

\end{tikzpicture}
}
	\caption{AL-STL Framework with Performance-Graph Advantage.}
	\label{fig:framework}
\end{figure*}

We now describe our proposed framework, shown in \autoref{fig:framework}, that closes the RL training loop to extract both the reward function and policy that optimally satisfy $\Phi$, resembling apprenticeship learning. The corresponding pseudocode is given in \autoref{alg:AL_STL}. Analogous to the replay buffer in RL, we introduce storage buffers for the reward model: (i) \textit{frontier} $\frontier$ containing the best episode rollouts of the agent so far and (ii) \textit{candidate} $\candidate$ containing the rollouts under the current reward and policy with PGAs. Initially, the frontier is populated with demonstrations (\autoref{alg:AL_STL:front_init}) from which the global DAG $\globalG$ and hence the reward function are extracted via PeGLearn (\autoref{alg:AL_STL:lfd_stl}). RL is performed with the learned rewards and each rollout is associated with its PGA, that is optimized either in the episode returns or in the loss. Upon updating the policy, multiple rollouts are collected in the candidate buffer (loop on \autoref{alg:AL_STL:get_candidate}), and the frontier is updated by comparing the overall PGAs of both the frontier and candidate based on a strategy (\autoref{alg:AL_STL:update_strat}) that we describe in \autoref{sec:update_strats}. This loop, shown by the yellow background in \autoref{fig:framework}, continues for a finite number of cycles or until the frontier can no longer be updated. At this stage, the reward and policy representing the frontier optimally satisfy $\Phi$, which we discuss in \autoref{sec:convergence}.

\begin{algorithm}[htbp]
\DontPrintSemicolon
\SetKwInput{Input}{Input}
\Input{
	$\demoset \defeq$ demonstration set; $\Phi \defeq$ specification set \\
	} 
\KwResult{
	$R^* \defeq$ reward function; $\pi^* \defeq$ a policy 
	}

\Begin{
	$\frontier \leftarrow \demoset$ \tcp*{Initialize frontier} \label{alg:AL_STL:front_init}
	$converged = \bot$ \;

	\While{$\neg converged$}{
		$R \leftarrow$ \texttt{PeGLearn}$(\frontier, \Phi)$ \tcp*{reward function from trajectories in $\frontier$} \label{alg:AL_STL:lfd_stl}
		
		$\candidate \leftarrow \emptyset$ \tcp*{Initialize the candidate}

		$\pi \leftarrow $ perform RL with \texttt{PGA} \; \label{alg:AL_STL:rl_pga}
		
		\tcp{Rollout $k$ trajectories from $\pi$ and add them to $\candidate$}
		\For{$i \leftarrow 1$ \KwTo $k$}{\label{alg:AL_STL:get_candidate}
			$\tau_i \leftarrow \langle (s_t, a_t \sim \pi(s_t)) \rangle_{t=0}^{T}$ \;
			$\candidate \leftarrow \candidate \cup \tau_i$ \; 
		}

		$converged \leftarrow \texttt{Update}(\candidate, \frontier$) \; \label{alg:AL_STL:update_strat}
	}
	\Return $R^*=R, \pi^*=\pi$ \;
}
\caption{STL-Guided Apprenticeship Learning\label{alg:AL_STL}}
\end{algorithm}

\subsubsection{Frontier Update Strategies}
\label{sec:update_strats}

$\frontier$ and $\candidate$ contain rollouts that are associated with their PGAs. We define an operator $\odot \in \{\texttt{min}, \texttt{max}, \texttt{mean}\}$, and therefore, the metrics $\Fmet \doteq \odot \{\texttt{PGA}(\tau) | \tau \in \frontier \}$ and $\Cmet \doteq \odot \{\texttt{PGA}(\tau) | \tau \in \candidate \}$. To update the frontier, we propose the {\em strategic merge} operation as: 
\begin{enumerate}[label=(\alph*)]
	\item We first compare whether $\Cmet > \Fmet$, i.e., the trajectories with the newly-explored PGA are better than the current best trajectories in $\frontier$. The operator $\odot$ acts as the criterion for filtering bad-performing trajectories.
	\item If so, we retain the trajectories in $\frontier \cup \candidate$ whose PGAs are greater than $\Fmet$ and discard the others; resulting trajectories form the new $\frontier$. Formally, this is given by $\frontier \leftarrow \{\tau | \texttt{PGA}(\tau) > \Fmet, \tau \in \frontier \cup \candidate\}$. That is, quality of the worst $\odot$ criteria-based rollouts in $\frontier$ is improved.
	\item Otherwise, $\frontier$ already has the best trajectories so far and is left unaltered. If the statistic $\odot$ for $\frontier$ and $\candidate$ are similar (i.e., their difference is below some threshold) upon sufficient exploration, then convergence is achieved.
\end{enumerate}

In theory, with unbounded memory, the frontier would be able to keep all the best-performing trajectories. For practical implementations, both buffers are bounded (say $p$), so we keep the top-$p$ trajectories in the frontier in our experiments. The {\em strategic merge} is not the only way to maintain the buffer, however, it offers some performance guarantees as we show in \autoref{sec:convergence}. One could consider a na\"ive approach of simply merging all the trajectories in both buffers without any filtering criteria. Alternately, one could also replace all the trajectories in $\frontier$ with those in $\candidate$, which also exhibits monotonic improvement in the RL policy.

\subsubsection{Policy Improvement Analysis}
\label{sec:convergence}

In order to analyze \autoref{alg:AL_STL} and show policy improvement, we make certain assumptions about the task and RL models:
\begin{enumerate}[label=(\alph*)]
	\item The specifications accurately represent the task. 
	\item The task can be completed, regardless of optimal behavior, with the given MDP configurations and task specifications. That is, our algorithm requires {\em at least one} demonstration that can satisfy all specifications, but is not required to be optimal.
	\item The RL model used to train the agent always has an active exploration component (stochastic policy or an exploration rate) to cover the MDP spaces. This not only helps in discovering new policies, but also helps learn more accurate reward models. Theoretically, with infinite timesteps, the RL agent will have fully explored the environment spaces to find the optimal policy \cite{sutton_reinforcement_2018}. In practice, the timesteps are set to a large finite value for majority coverage of the spaces.
\end{enumerate}

Here, we describe how the {\em strategic merge} functionality exhibits policy improvement. From \autoref{sec:update_strats}, the new $\frontier$ contains the set of trajectories given by $\frontier = \{\tau | \texttt{PGA}(\tau) > \Fmet, \tau \in \frontier \cup \candidate\} $. For the purpose of this proof, we will consider $\odot$ to be the \texttt{mean}. Then, $\Fmet = \frac{\sum_{\tau \in \frontier} \pga(\tau)}{\abs{\frontier}}$ and $\Cmet = \frac{\sum_{\tau \in \candidate} \pga(\tau)}{\abs{\candidate}}$. We know that the $\frontier$ is updated in the $\texttt{Update}$ function when $\Cmet > \Fmet$. Let $\Fmet'$ be the mean of the intermediate set $\frontier' = \frontier \cup \candidate$. Then,
\begin{align}
	\Fmet' &= \frac{\sum_{\tau \in \frontier'} \pga(\tau)}{\abs{\frontier'}} =  \frac{\sum_{\tau \in \frontier} \pga(\tau) + \sum_{\tau \in \candidate} \pga(\tau)}{\abs{\frontier} + \abs{\candidate}} \nonumber \\
	 &= \frac{\abs{\frontier} \Fmet + \abs{\candidate} \Cmet}{\abs{\frontier} + \abs{\candidate}} = \Fmet + \frac{\abs{\candidate}k}{\abs{\frontier} + \abs{\candidate}} 
	\label{eq:merge}
	\intertext{since $\Cmet > \Fmet$, we can write this as $\Cmet = \Fmet + k$, where $k > 0$} \nonumber
\end{align}
Now, let $\Fmet''$ be the new mean after filtering $l < (\abs{\frontier} + \abs{\candidate})$ rollouts whose $\pga \leq \Fmet$ in the merged set $\frontier'$. 
\begin{align}
	\Fmet'' &= \frac{\abs{\frontier'}\Fmet' - \Sigma \{\pga(\tau) |\tau \in \frontier', \pga(\tau) \leq \Fmet\}}{\abs{\frontier'}-l} \nonumber \\
	\intertext{In the worst case, all $l$ trajectories have PGAs at most $\Fmet$.} \nonumber
	\Fmet'' &\geq \frac{\abs{\frontier'}\Fmet' - l\Fmet}{\abs{\frontier'}-l} = \frac{(\abs{\frontier} + \abs{\candidate})\Fmet' - l\Fmet}{\abs{\frontier} + \abs{\candidate} - l} \nonumber \\
	\Fmet'' &\geq \Fmet + \frac{\abs{\candidate}k}{\abs{\frontier} + \abs{\candidate} - l} \label{eq:discard} \quad \text{(substituting from \autoref{eq:merge})}
\end{align}
As the cardinalities of both buffers $\frontier$ and $\candidate$ are non-zero, the denominator $(\abs{\frontier} + \abs{\candidate} - l) > 0$. Thus, in \autoref{eq:discard}, the second term is always positive, which proves that our algorithm improves the policy and reward in each cycle, under the exploration assumption. A special case of \autoref{eq:discard} is when $\frontier$ is completely replaced by $\candidate$, i.e., when all $l$ trajectories belong to $\frontier$, then $l=\abs{\frontier}$ and so, $\frontier$ inherits the higher mean from $\candidate$. The frontier remains unchanged when the demonstration set is optimal or the rollouts in $\frontier$ at the end of each training cycle are optimal.

We can apply similar reasoning to the other operators for $\odot$. In the case of $\texttt{max}$, the frontier's maximum value will always inherit the maximum (i.e., the best rollouts) from the candidate. For $\texttt{min}$, only the least-performance trajectories are discarded and the second-to-least ones are updated to be the new minimum in $\frontier$. Since the upper-bound of $\frontier$ is $n\Delta$, our method keeps improving the policy towards this maximum. {\em However, this does not guarantee that the maximum value can always be achieved due to several factors: conflicting specifications causing trade-offs, environment configuration, solvability of the MDP under the given specifications, etc.}

\subsubsection{Effect of Affine Transformations to Rewards}
\label{sec:policy_invariance}

In practice RL is sensitive to the hyperparameter settings, environment stochasticity, scales of rewards and observations, and other algorithmic variances \cite{islam2017reproducibility}. Hence, in our experiments, we normalize observations and rewards using affine transforms. However, applying affine transformations to the reward function does not alter the optimal policy\cite{ng_policy_invariance}. We also prove this for basic scaling and shifting of the rewards by a constant amount in \autoref{app:proofs}.

\section{Experiments}
\label{sec:exps}

Our proposed framework is evaluated on a diverse set of PyBullet or MuJoCo physics simulation-based robot manipulation tasks (\autoref{fig:envs}): (i) reaching a desired pose with the end-effector, (ii) placing an object at a desired location, and (iii) opening doors. In all our experiments, the task specifications only monitor the observed states and so, the rewards are a function of just the states. The STL specifications are evaluated using \texttt{RTAMT} \cite{rtamt_stl}. The reward function is modeled by regression with either fully connected neural networks or Gaussian processes, implemented in PyTorch. Our framework is based on the Stable-Baselines3\footnote{\url{https://github.com/DLR-RM/stable-baselines3}\label{ft:sb3}} implementations of RL algorithms. All simulations are performed on an Ubuntu desktop with an {\em Intel\textregistered Xeon 8-core CPU} and {\em Nvidia Quadro RTX 5000 GPU}. For each environment, $m=5$ demonstrations are generated by training an appropriate RL agent under an expert dense reward function. {\em In these domains, every RL episode features a unique/randomized target and hence the collected demonstrations are also unique (i.e., the states do not overlap. Additionally, these simulations implicitly model noise in the environment which make it challenging to provide optimal trajectories.)} Details of all hyperparameters can be found in \autoref{app:experiments}. In all tasks, unless explicitly stated, the frontier is updated by completely replacing its contents with the candidate (i.e., special case of \autoref{eq:discard}) and we set $\abs{\frontier} = \abs{\candidate} = 5$.

\begin{figure}[htbp]
    \centering
    \subfloat[Panda Pose Reaching]{\includegraphics[width=0.4\linewidth, height=1.2in]{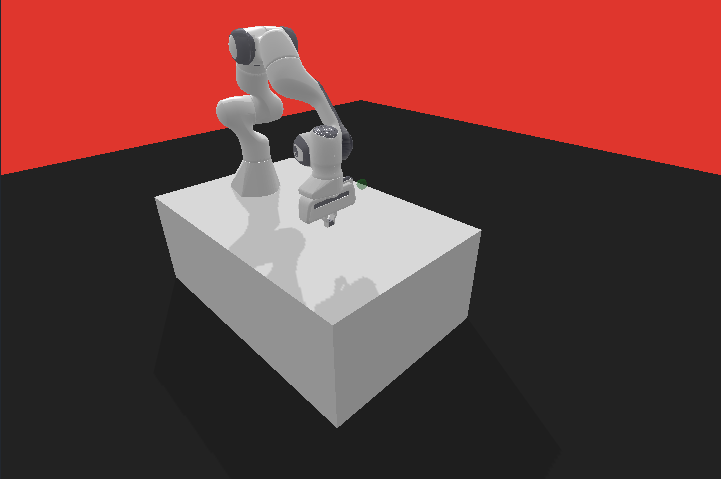}} \hfil
    \subfloat[Needle Pose Reaching]{\includegraphics[width=0.4\linewidth, height=1.2in]{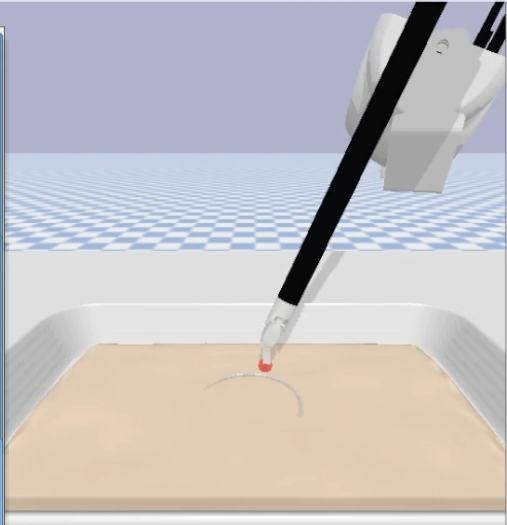}}

    \subfloat[Picking and Placing]{\includegraphics[width=0.4\linewidth, height=1.2in]{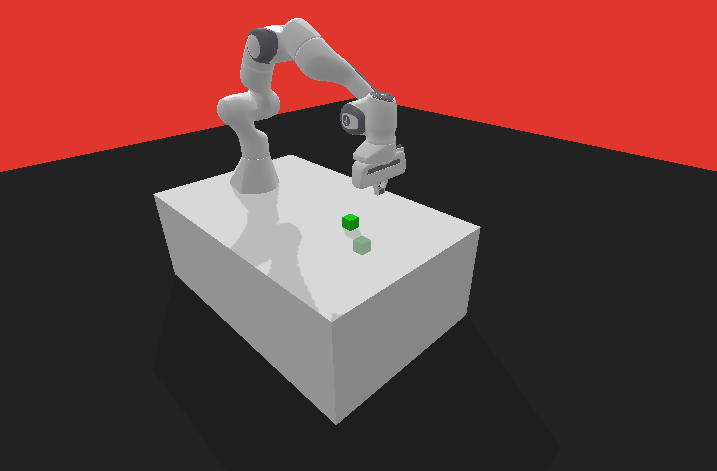}} \hfil
    \subfloat[Door Opening]{\includegraphics[width=0.4\linewidth, height=1.2in]{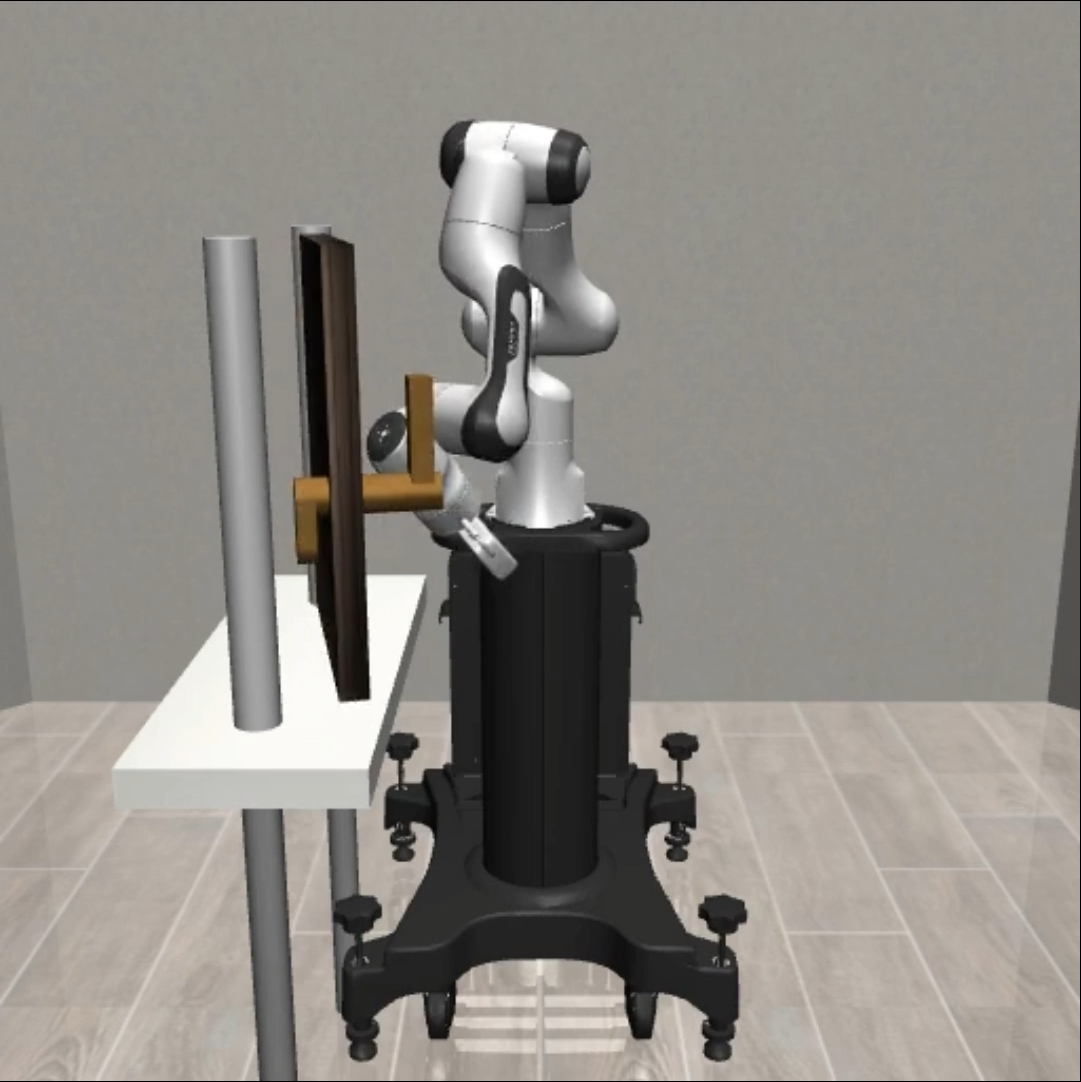}}
    \caption{Overview of the robot manipulator environments.}
    \label{fig:envs}
\end{figure}

\subsection{Task - Reaching Pose}

The end-effector of a Franka Emika Panda robot \cite{pandagym} is required to reach the target pose as quickly as possible, the specifications for which are given as: $\varphi_1 \defeq \ev(d < \delta)$ and $\varphi_2 \defeq \alw(t < T)$, where $d$ is the $l^2$-norm of the difference between the end-effector and target poses, $\delta$ is a small threshold to determine success, and $T$ is the desired time in which the target must be achieved. For evaluation on a more precise environment, we use a surgical robot environment - SURROL \cite{surrol} that is built on the da Vinci Surgical Robot Kit \cite{dVRK}. In this common surgical task, a needle is placed on a surface and the goal is to move the end-effector towards the needle center. The specifications for this task follow the same template above, however, the threshold is very small, i.e., $\delta = 0.025$, requiring highly precise movements.

The rewards were modeled with a 2-layer neural network and scaled to $[-1,1]$. The RL agent used SAC \cite{sac} with hindsight experience replay (HER) \cite{her_replay} and was trained for 5 cycles spanning $2\cdot 10^5$ timesteps for Panda-Reach and $2.5\cdot 10^5$ steps for Needle-Reach. To validate reproducibility, the training and evaluation was performed over 5 random seeds using the same 5 demonstrations. The results for both these environments are shown in \autoref{fig:reaching_tasks}. The first column shows the PGA over time or cycles (note the scale of $y$-axis). The learned policies in both environments achieve have PGA $\approx 2$ since there are 2 specifications. The second column represents the specification weights. In the surgical task, the final weights are uniform as desired since the room for error in completing the task is very small, while the Panda task has a larger threshold for completion which affects the resolution of the smooth STL semantics, though all tasks are completed successfully.

\begin{figure}[htbp]
    \centering
    \subfloat[Panda-Reach\label{fig:pandareach}]{\includegraphics[width=\linewidth, height=1.4in]{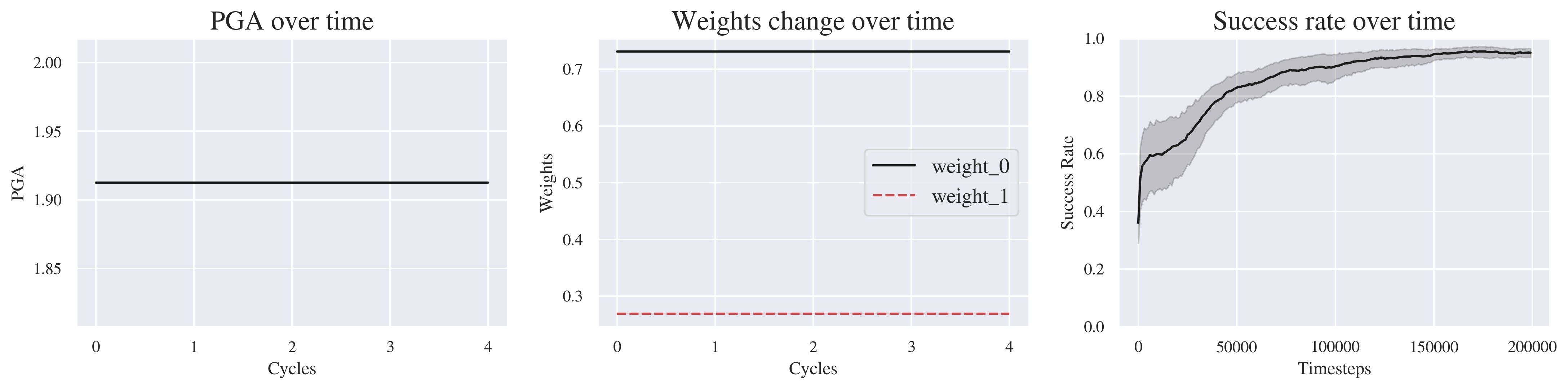}}

    \subfloat[Needle-Reach\label{fig:needlereach}]{\includegraphics[width=\linewidth, height=1.4in]{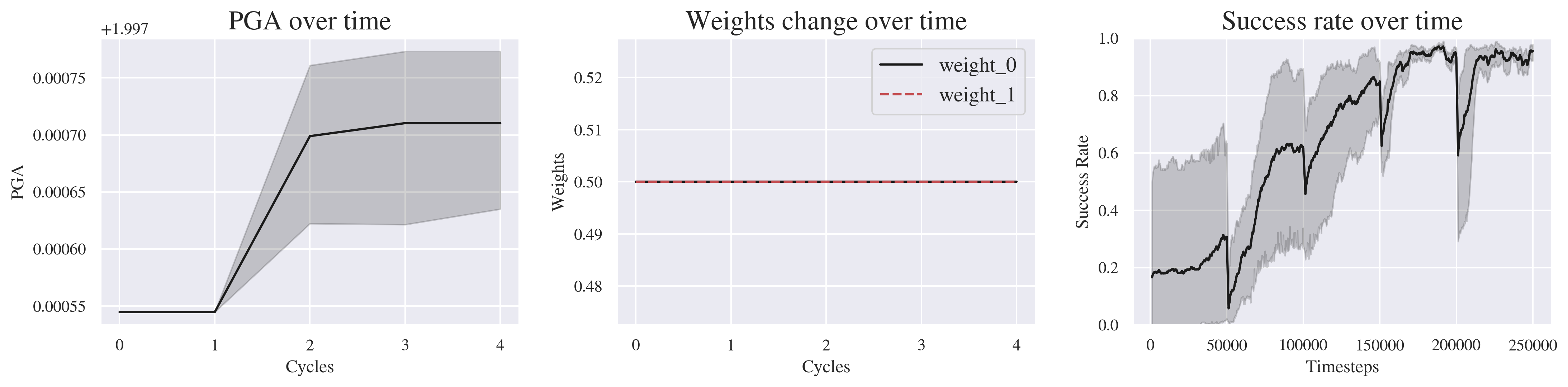}} 
    \caption{Summary of training and evaluations for the pose-reaching tasks.}
    \label{fig:reaching_tasks}
\end{figure}

In both tasks, using just 5 demonstrations, AL-STL achieved over 99\% mean success rate in both, training (right figures) and evaluations; 5 random seeds were used for evaluations. For Needle-Reach, the baselines \cite{surrol, surrol_baseline} that used BC and IRL, required 100 expert demonstrations. It is shown in \cite{surrol_baseline} that, when the number of demonstrations is reduced to just 10, which is still \textbf{2x} larger than ours, the success rate drops drastically. For Panda-Reach, the authors of \cite{haldar2022watch} show that imitation learning outperforms adversarial IRL techniques when each method uses 50 demonstrations, though both eventually learn to succeed in the task. This however is still \textbf{10x} more than the amount of samples required by our work.

\subsection{Task - Placing Cube}

Here, a Panda robot is required to pick up a cube on a table and place it at the desired location\cite{pandagym}. Only 4 of the 5 demonstrations were successful. The specifications are also similar to the Reach tasks, i.e., eventually the distance between the cube and desired pose is below a threshold and the robot must do so as quickly as possible. The reward function was modeled with Gaussian process and scaled in $[-3, 3]$, for better visualization (note that scaling does not affect the optimal policy as described in \autoref{sec:policy_invariance}). The RL agent used TQC \cite{tqc} with HER and was trained for $10^7$ timesteps distributed across 5 cycles, and achieved a training success rate of $98\%$ (\autoref{fig:panda_pandp}). The resulting policy was evaluated over 20 episodes across 5 random seeds (=100 test configurations), achieving a success rate of $\mathbf{96\%}$ in the test trials. From \autoref{fig:panda_pandp}, we see that the algorithm converges to a high success rate after just 3 cycles.

The task specification is significantly challenging because it only describes that the cube be placed at the desired pose. In other words, the RL agent must learn to reach the cube, grasp and move to the desired location while holding the cube. It must learn this sequence of elementary behaviors from just the handful of demonstrations provided. Another remarkable finding in our work (shown in the supplemental video), is that the policy learns to (i) correctly \textit{pick} the cube and place it at the target whenever the target height is above the table and (ii) {\em push/drag} the cube when the target is on the same table surface. This shows that our algorithm combines RL exploration and graph advantage to possibly learn {\em acceptable} behaviors that were not observed before. The number of demonstrations used for this task in the baselines that achieved comparable success rates, are: 100 for MCAC \cite{mcac}, between 4 and 16 for OPRIL \cite{opril}, 20 for goalGAIL \cite{goalGail} and 50 for ROT \cite{haldar2022watch}.

\subsection{Task - Opening Door}

A Panda robot, mounted on a pedestal, is required to open a door\cite{robosuite}. Only 3 of the 5 demonstrations were successful. The task is successful if the door hinge is rotated beyond $\theta=0.3$rad. The task specifications consist of (a) reaching the door handle and (b) rotating the hinge beyond $\theta$. The elementary behaviors to be learned are: reaching the door handle (similar to the target pose-reaching tasks), turning the handle to unlock the door and pulling to open the door. This is a non-trivial task for expert reward design as it must capture all these elementary behaviors and compose them sequentially. The reward function for AL-STL was modeled with a 3-layer neural network. Since this is a more challenging task, the frontier was updated with {\em strategic merge}, and the size of reward buffers were set to $20$ to collect more rollouts. The RL agent used TQC and was trained for $5\cdot10^6$ steps across 25 cycles to achieve a success rate of $\mathbf{98\%}$ (\autoref{fig:panda_door}). The resulting policy was evaluated over 20 episodes across the same 5 test seeds (=100 new test configurations) used in \cite{robosuite} for comparison purposes. In the test scenarios, it achieved a success rate of $\mathbf{100\%}$.

\begin{figure}[htbp]
    \centering
    \subfloat[Cube-Placing\label{fig:panda_pandp}]{\includegraphics[width=0.5\linewidth,height=1.7in]{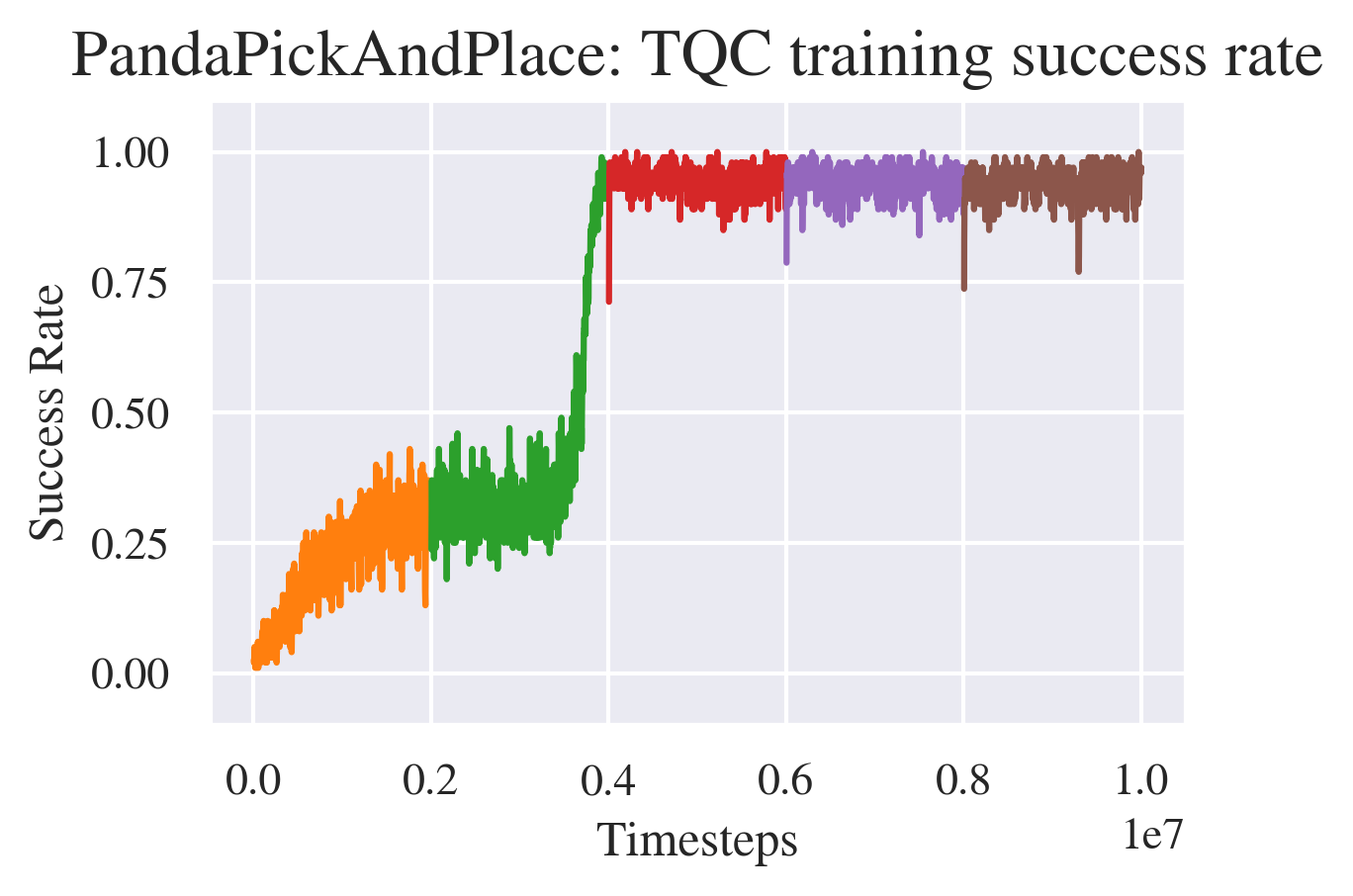}} 
    \subfloat[Door Opening\label{fig:panda_door}]{\includegraphics[width=0.5\linewidth,height=1.7in]{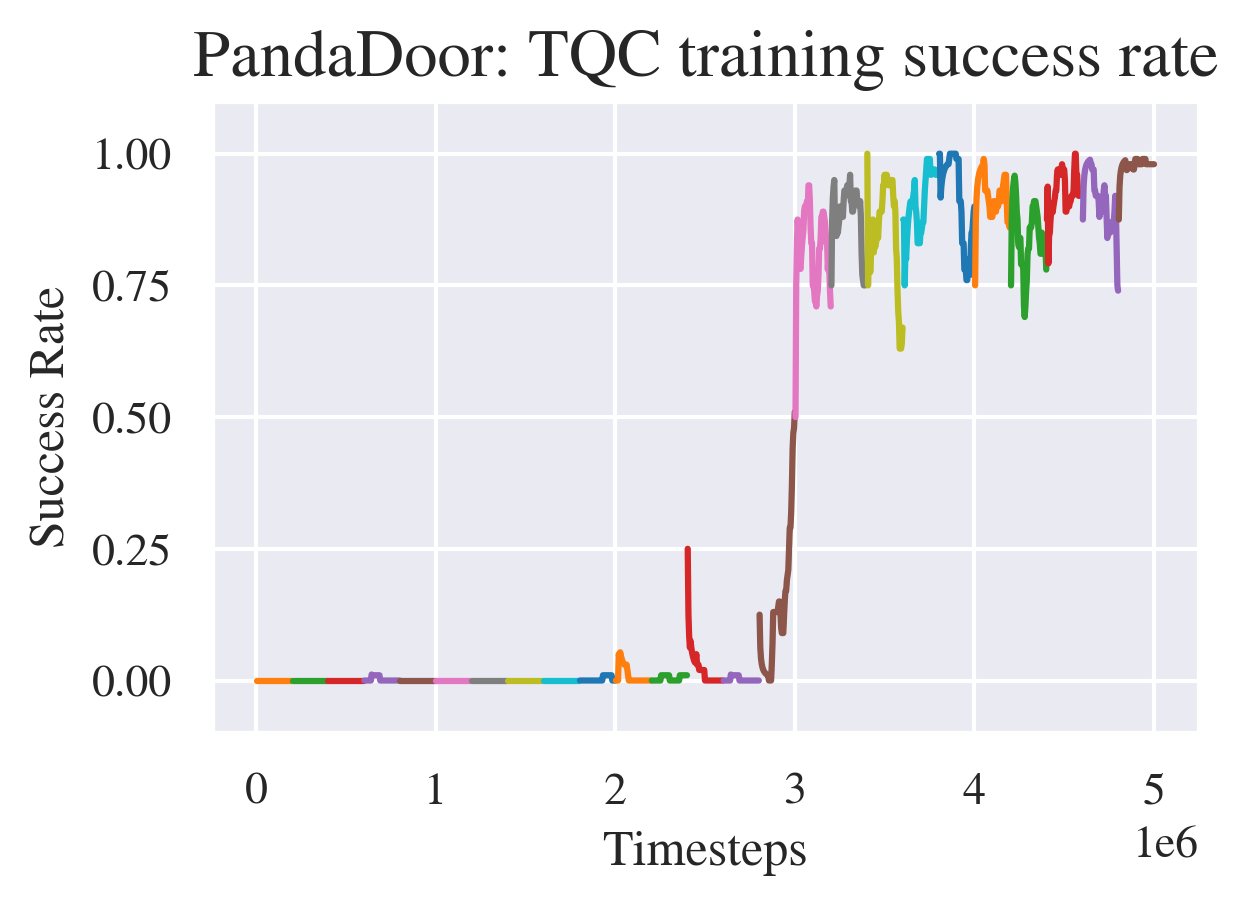}} 
    \caption{Summary of RL training for the object manipulation tasks.}
    \label{fig:panda_pandp_door}
\end{figure}

We compare our work with two state-of-the-art baselines MCAC \cite{mcac} and OPIRL \cite{opril}, which has shown to outperform maximum entropy and adversarial IRL-based methods. While both these methods successfully complete this task, MCAC used 100 demonstrations, while OPIRL used between 4 and 16. OPIRL had significantly more variance (i.e., unstable learning) with 4 demonstrations compared to using 16/ Furthermore, in OPRIL, the method uses a substantially large reward buffer size of $2\cdot 10^6$ to compensate for the limited demonstrations, while ours uses $2\cdot 10^4$ (i.e., $\abs{\frontier}=\abs{\candidate}=20$, each trajectory of length 500). This indeed shows our method is more efficient compared to IL and IRL. 

A summary of comparisons with baselines based on sample complexity for all environments is given in \autoref{tab:comparisons}. The baselines are the methods in each category of IL and IRL, as discussed above, that successfully complete the task and require as few demonstrations as possible.
\begin{table}[!t]
	\caption{Comparisons with baselines for environments.\label{tab:comparisons}}
	\centering
    \begin{tabular}{cccc}
      \toprule
      TASK &  \multicolumn{3}{c}{\# DEMOS}\\
      & OURS & IRL & IL (BC) \\
      \midrule
      Panda Reach & 5 & 50 & 50 \\
      Needle Reach & 5 & 100 & 100 \\
      Panda Pick-And-Place & 5 & 16 & 20 \\
      Panda Door & 5 & 16 & 100 \\
      \bottomrule
	\end{tabular}
\end{table}

\section{Conclusion}

We proposed a novel extension to the LfD-STL framework by introducing closed-loop learning to infer both the rewards and control policy simultaneously. We proposed a graph-based formalism to provide a succinct representation of multiple non-Markovian (temporal) task specifications. These graphs provided quantitative and interpretable assessments of agent behaviors. The graph advantage metric guided the agent's learning process to maximally satisfy the task specifications and perform optimal trade-offs. Through experiments on various robotic manipulation domains, we have shown the effectiveness of our approach in terms of sample efficiency and discussed how it outperforms several state-of-the-art methods. For future, we propose to investigate (i) diversity in demonstrations and how it affects the learning process and solvability of MDPs, (ii) probabilistic guarantees and verification, and (iii) task generalization and sim2real transfer-learning.

\clearpage


\bibliography{surveys,learning,cps,tools}  

\clearpage
\section*{Appendix}
\appendix
\section{Signal Temporal Logic}
\label{app:stl}

\begin{definition}[Quantitative Semantics for Signal Temporal Logic]%
\label{def:quantitative}
    Given an algebraic structure $(\oplus,\otimes,\top,\bot)$, we define the
    quantitative semantics for an arbitrary signal $\sig$ against an STL formula
    $\varphi$ at time $t$ as in \autoref{tab:stl_quant}.
	\begin{table}[htbp]
	\caption{Quantitative Semantics of STL}
	\label{tab:stl_quant}
	\centering
    \begin{tabular}{cc}
      \toprule
      $\varphi$ &  $\robustness{\varphi}{t}$ \\
      \midrule
      $\mathit{true}$/$\mathit{false}$ & $\top$/$\bot$ \\
      $\mu$                       & $f(\sig(t))$ \\
      $\neg \varphi$              & $-\robustness{\varphi}{t}$ \\
      
      $\varphi_1 \wedge \varphi_2$ &
      $\otimes(\robustness{\varphi_1}{t},\robustness{\varphi_2}{t})$ \\
      
      $\varphi_1 \vee \varphi_2$ & 
      $\oplus(\robustness{\varphi_1}{t},\robustness{\varphi_2}{t})$ \\
      
      $\alw_I(\varphi)$ &  $\otimes_{\tau\in t+I}(\robustness{\varphi}{\tau})$ \\
      
      $\ev_I(\varphi)$ &  $\oplus_{\tau\in t+I}(\robustness{\varphi}{\tau})$ \\

      $\varphi \until_I \psi$ & $\oplus_{\tau_1\in t+I}
      (\otimes(\robustness{\psi}{\tau_1},\otimes_{\tau_2\in[t,\tau_1)}(\robustness{\varphi}{\tau_2}))$ \\
      \bottomrule

	\end{tabular}
\end{table}
\end{definition}

A signal satisfies an STL formula $\varphi$ if it is satisfied at time $t=0$. Intuitively, the quantitative semantics of STL represent the numerical distance of ``how far'' a signal is away from the signal predicate. For a given requirement $\varphi$, a demonstration or policy $d$ that satisfies it is represented as $d \models \varphi$ and one that does not, is represented as $d \not\models \varphi$. In addition to the Boolean satisfaction semantics for STL, various researchers have proposed quantitative semantics for STL, \cite{fainekos_robustness_2009,jaksic_quantitative_2018} that compute the degree of satisfaction (or \textit{robust satisfaction values}) of STL properties by traces generated by a system. In this work, we use the following interpretations of the STL quantitative semantics: $\top = +\infty$, $\bot = -\infty$, and $\oplus = \max$, and $\otimes = \min$, as per the original definitions of robust satisfaction proposed in \cite{fainekos_robustness_2009,donze_robust_2010}.

\section{Derivations and Proofs}
\label{app:proofs}

As mentioned in the main paper, we show that applying affine transformations to the reward function do not change the optimal policy. Particularly, we are concerned with scaling and shifting the rewards by a constant factor.

\begin{lemma}
    The optimal policy is invariant to affine transformations in the reward function.
    \label{lemma}
\end{lemma}

\begin{proof}[Proof Sketch]
    From \cite{sutton_reinforcement_2018}, we have the definition of the $Q$ function as follows, for the untransformed reward function $R$: 
    \begin{align}
        Q(s,a) &\doteq \Expect \left[\sum_{k=0}^{\infty}\gamma^k \cdot R(s,a)_{t+k+1} | S_t=s, A_t=a\right] \label{eq:returns}\\
        Q(s,a) &\doteq R(s,a) + \gamma \sum_{s'}P(s,a,s') \max_{a'}Q(s',a') \label{eq:bellman}
    \end{align}

    We consider two cases of reward function affine transformations in our work: (a) scaling by a positive constant and (b) shifting by a constant. In both these cases, our objective is to express the new $Q$ function in terms of the original. Note that we abbreviate $R(s,a)$ to just $R$ for simplicity.

    \paragraph*{Case (a): Scaling $R$ by a positive constant}
    Let the scaled reward function be defined as $R' = c \cdot R, c > 0$. The new $Q$ function is then
    \begin{align*}
        Q'(s,a) &\doteq \Expect \left[\sum_{k=0}^{\infty}\gamma^k \cdot R'_{t+k+1} | S_t=s, A_t=a\right] \\
        Q'(s,a) &= \Expect \left[\sum_{k=0}^{\infty}\gamma^k \cdot c \cdot R_{t+k+1} | S_t=s, A_t=a\right] \\
        Q'(s,a) &= c \cdot \Expect \left[\sum_{k=0}^{\infty}\gamma^k \cdot R_{t+k+1} | S_t=s, A_t=a\right] \\
        Q'(s,a) &= c \cdot Q(s, a)
    \end{align*}
    Thus we see that the new $Q$ function scales with the scaling constant.

    From \autoref*{eq:bellman} and by later substituting for $Q'$ from the above result, we have,
    \begin{align*}
        Q'(s,a) &\doteq R'(s,a) + \gamma \sum_{s'}P(s,a,s') \max_{a'}Q'(s',a') \\
        c \cdot Q(s,a) &= c \cdot R(s,a) + \gamma \sum_{s'}P(s,a,s') \max_{a'}(c \cdot Q(s',a'))\\
        c \cdot Q(s,a) &= c \cdot R(s,a) + c\gamma \sum_{s'}P(s,a,s') \max_{a'} \cdot Q(s',a')\\
        Q(s,a) &= R(s,a) + \gamma \sum_{s'}P(s,a,s') \max_{a'} \cdot Q(s',a')
    \end{align*}
    Thus the Bellman equation holds indicating that the policy is invariant to scaling by a positive constant.

    \paragraph*{Case (b): Shifting $R$ by a constant}
    Let the shifted reward function be defined as $R' = R+c$. The new $Q$ function is then
    \begin{align*}
        Q'(s,a) &\doteq \Expect \left[\sum_{k=0}^{\infty}\gamma^k \cdot R'_{t+k+1} | S_t=s, A_t=a\right] \\
        Q'(s,a) &= \Expect \left[\sum_{k=0}^{\infty}\gamma^k \cdot (R_{t+k+1}+c) | S_t=s, A_t=a\right] \\
        Q'(s,a) &= \Expect \left[\sum_{k=0}^{\infty}\gamma^k \cdot R_{t+k+1} | S_t=s, A_t=a\right] + \sum_{k=0}^{\infty}\gamma^k c \\
        Q'(s,a) &= Q(s,a) + \frac{c}{1-\gamma}
    \end{align*}
    Thus we see that the new $Q$ values get shifted by the constant.

    From \autoref*{eq:bellman} and by later substituting for $Q'$ from the above result, we have,
    \begin{align*}
        Q'(s,a) &\doteq R'(s,a) + \gamma \sum_{s'}P(s,a,s') \max_{a'}Q'(s',a') \\
        Q(s,a) + \frac{c}{1-\gamma}  &= R(s,a) + c \\ &+ \gamma \sum_{s'}P(s,a,s') \max_{a'}\left(Q(s',a') + \frac{c}{1-\gamma}\right)\\
        Q(s,a) + \frac{c}{1-\gamma}  &= R(s,a) + c \\ &+ \gamma \sum_{s'}P(s,a,s') \max_{a'}Q(s',a') \\ &+ \gamma \sum_{s'}P(s,a,s') \frac{c}{1-\gamma}\\
        Q(s,a) + \frac{c}{1-\gamma}  &= R(s,a) + \gamma \sum_{s'}P(s,a,s') \max_{a'}Q(s',a') \\ &+ c + \frac{c\gamma}{1-\gamma} \\
        Q(s,a) &= R(s,a) + \gamma \sum_{s'}P(s,a,s') \max_{a'} \cdot Q(s',a')
    \end{align*}
    Thus the Bellman equation holds indicating that the policy is invariant to shifting by a constant.

\end{proof}

Therefore, any combination of scaling or shifting does not affect the optimal policy in our work. Similarly, the optimal policy is shown to be invariant towards reward shaping with potential functions \cite{ng_policy_invariance}.

\section{Experiment Details}
\label{app:experiments}

This section describes additional details about the experiments such as the STL task specifications, hyperparameters, training and evaluation results.

\subsection{Task - Discrete-Space Frozenlake}
We make use of the $Frozenlake$ (FL) deterministic environments from OpenAI Gym \cite{openai_gym} that consist of a grid-world of sizes 4x4 or 8x8 with a reach-avoid task. Informally, the task specifications are (i) eventually reaching the goal, (ii) always avoid unsafe regions and (iii) take as few steps as possible. In these small environments $m=5$ demonstrations of varying optimality are manually generated. We use A2C as the RL agent and show the training results in \autoref{fig:frozenlake}. The left figures show the statistics of the rollout PGAs and the evolution of weights over time. The right figures show the rewards accumulated and episode lengths.

We see from the left figures, that initially, the non-uniform weights of specifications correspond to the suboptimal demonstrations. And over time, the weights all converge to $1/3$ indicating that there are no edges in the final DAG, while the PGAs of rollouts from the final policy are maximum, as hypothesized. Since the environments are deterministic, the final policy achieve a 100\% success rate. Since the task can be achieved even with IRL-based methods, we compare the amount of demonstrations required. Under identical conditions, the minimum number of demonstrations used by MCE-IRL are 50 for 4x4 grid and 300 for 8x8 grid. The algorithm in \cite{wenchao_safety_al} uses over 1000 demonstrations in the 8x8 grid, even though they use temporal logic specifications similar to ours. {\em This clearly suggests that the choice of the reward inference algorithm plays a significant role in sample complexity}. This is due to the unsafe regions being scattered over the map, requiring the desirable {\em dense} features to appear very frequently. 

\begin{figure}[htbp]
    \centering
    \subfloat[FL4x4 Weights]{\includegraphics[width=0.5\linewidth]{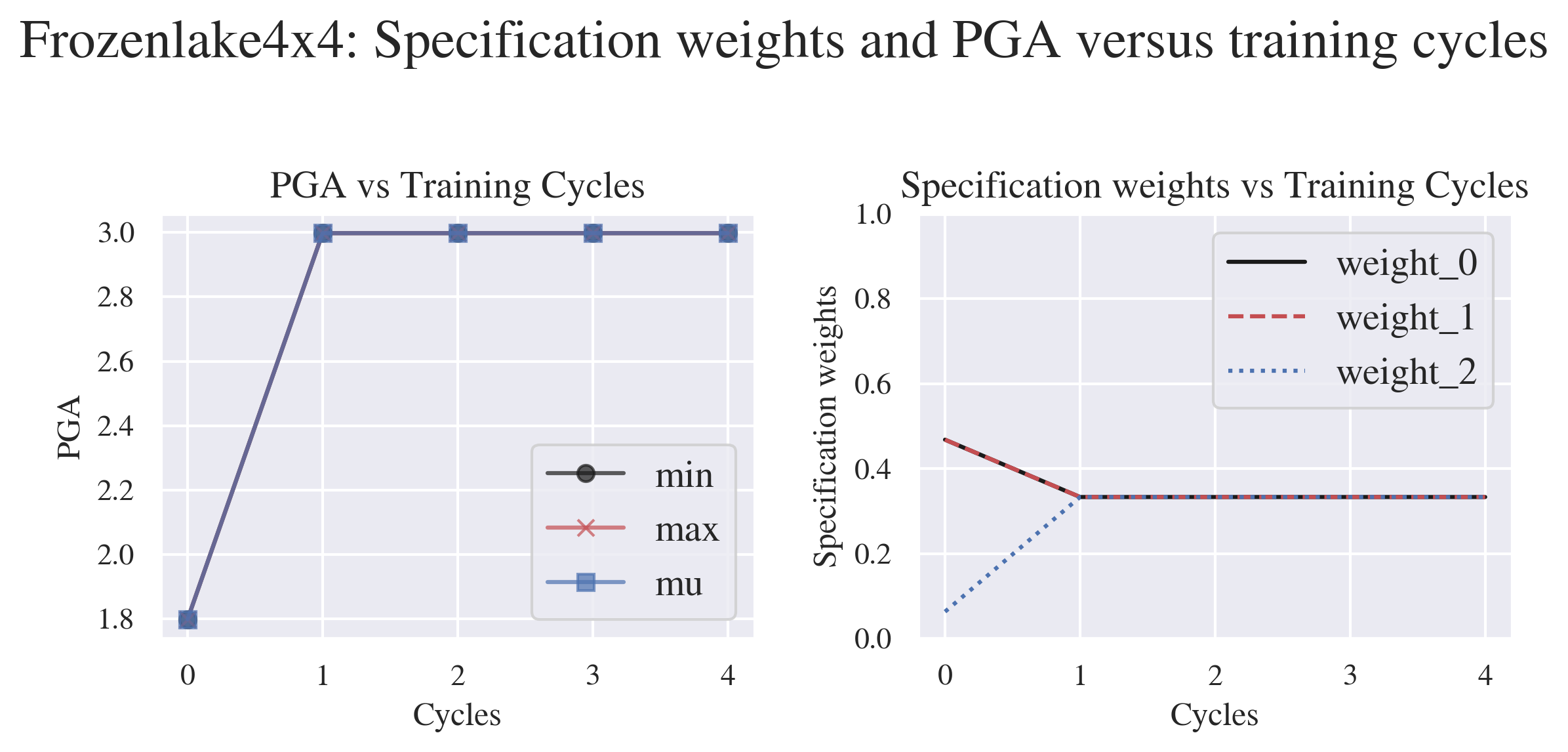}}
    \subfloat[FL4x4 Training Summary]{\includegraphics[width=0.5\linewidth]{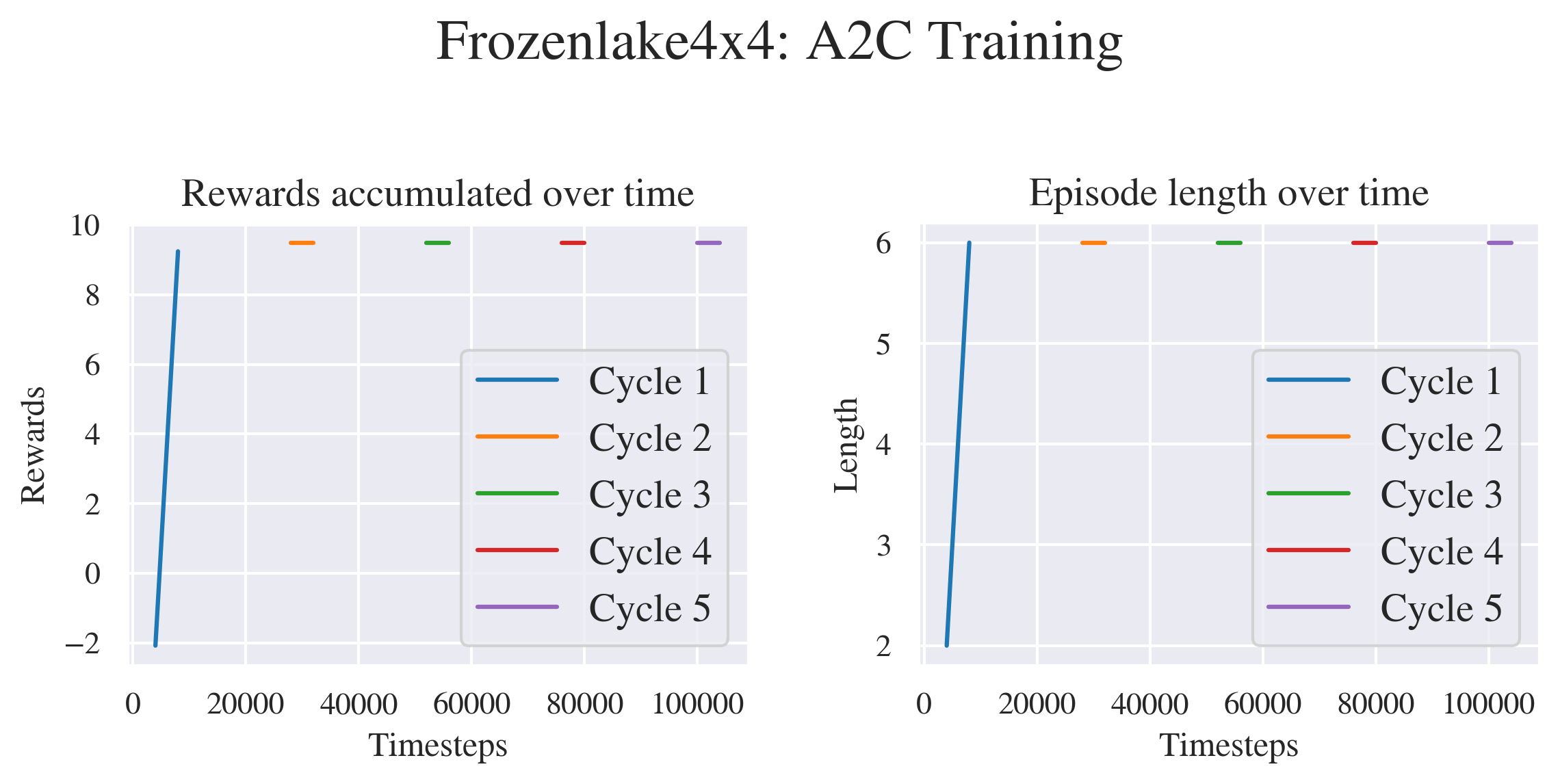}} 

    \subfloat[FL8x8 Weights]{\includegraphics[width=0.5\linewidth]{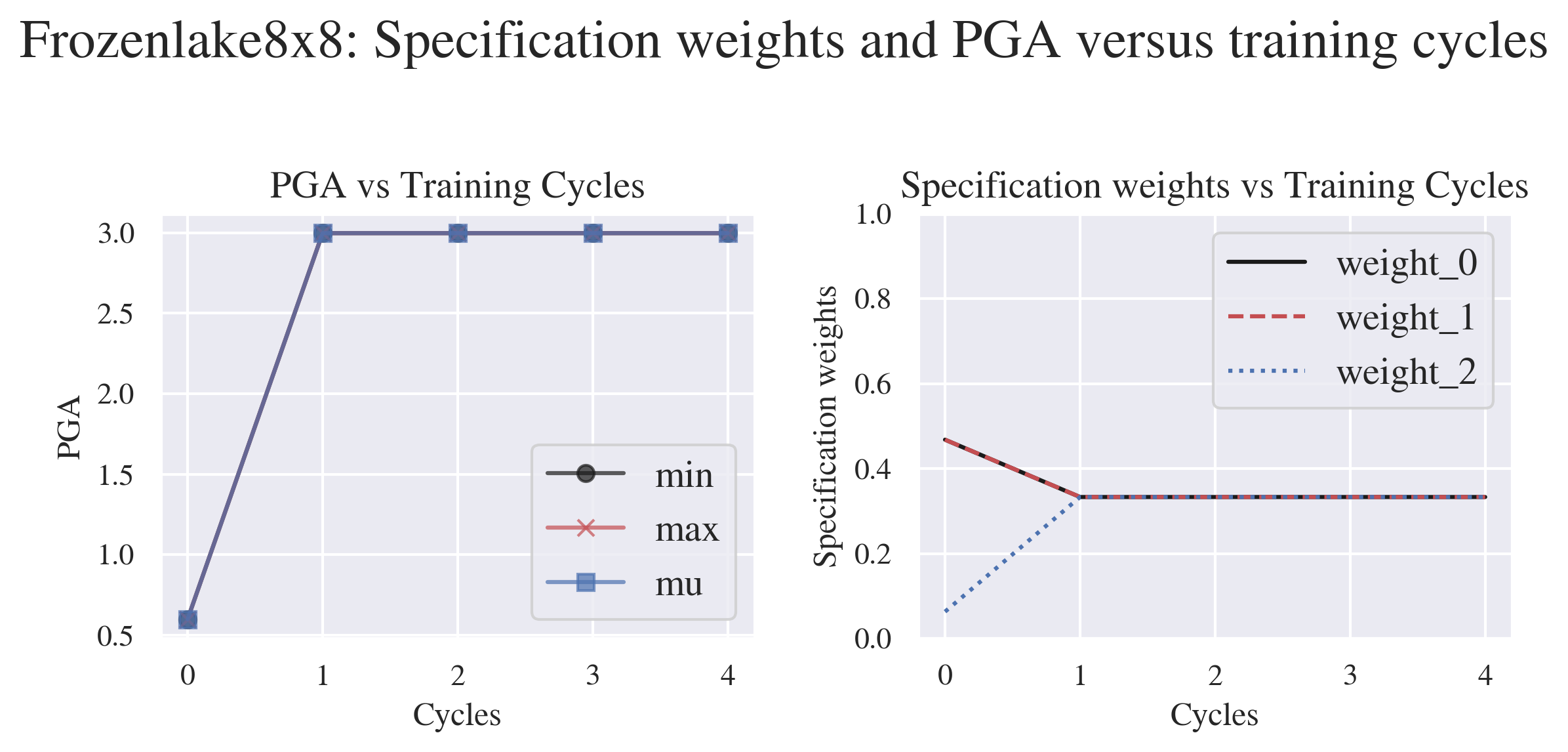}}
    \subfloat[FL8x8 Training Summary]{\includegraphics[width=0.5\linewidth]{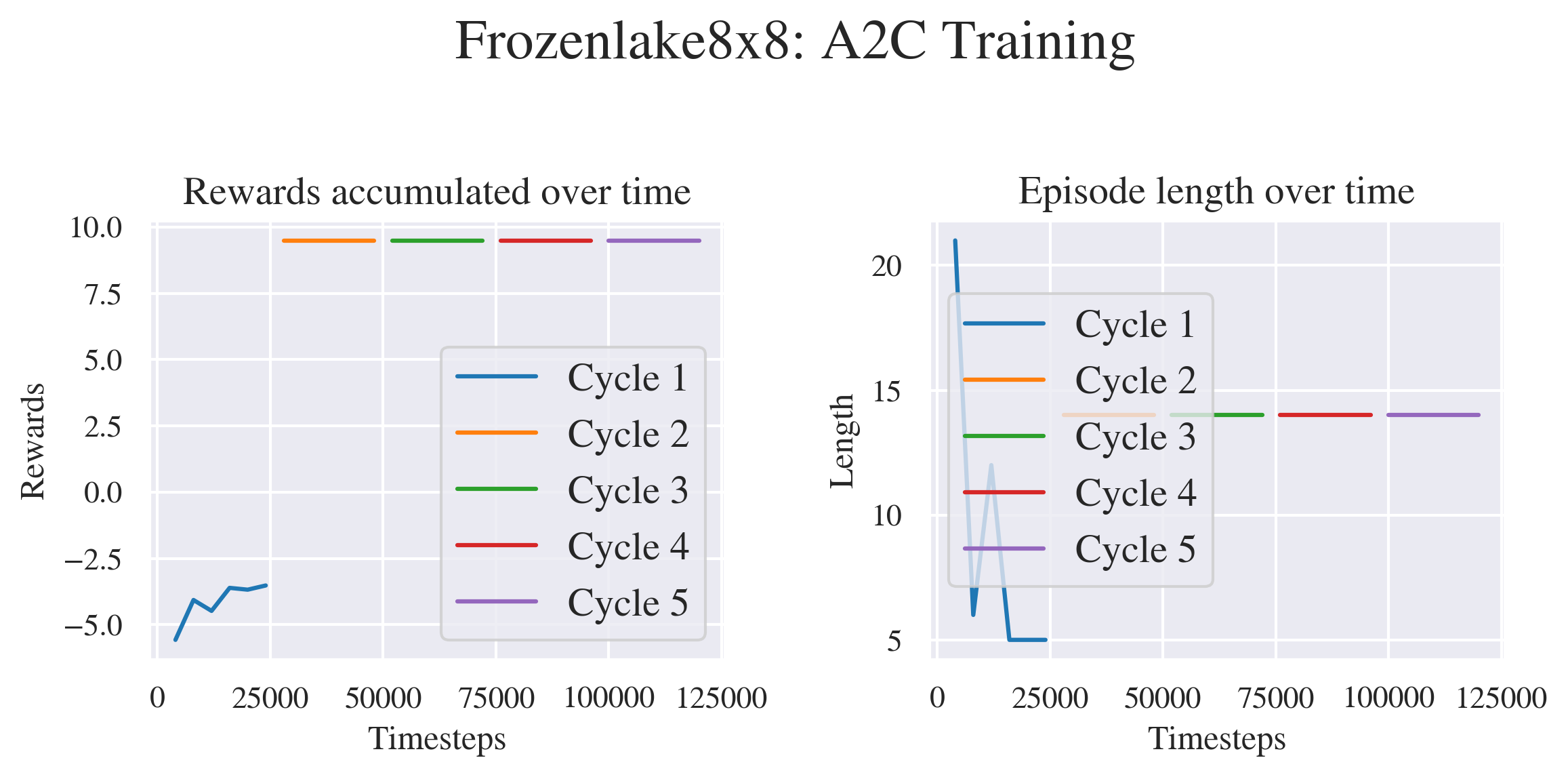}}
    \caption{Results for the 4x4 and 8x8 Frozenlake environments.}
    \label{fig:frozenlake}
\end{figure}

\subsection{Task - Reaching Pose}
The hyperparameters for both tasks: Panda-Reach and Needle-Reach, were nearly identical(\autoref{tab:reach_params}). The specifications for both these tasks are:
\begin{enumerate}
    \item Reaching the target pose: $\phix{1} \defeq \ev(\norm{ee_{pose}-target_{pose}} \leq \delta)$, where $ee$ indicates the end-effector and $\delta$ is the threshold used to determine success. For Panda-Reach, $\delta = 0.2$ and for Needle-Reach, $\delta=0.025$. 
    \item Reaching the target as quickly as possible: $\phix{2} \defeq \alw(t <= 50)$, where $t$ is the time when the end-effector reaches the target.
\end{enumerate}

\begin{table}[htbp]
	\caption{Reach task hyperparameters.\label{tab:reach_params}}
	\centering
    \begin{tabular}{c|cc}
      \toprule
      Parameters & \multicolumn{2}{c}{Values} \\
      & Panda-Reach & Needle-Reach \\
      \midrule
      \# Demos & \multicolumn{2}{c}{5} \\
      Reward Model & \multicolumn{2}{c}{Neural Network $[200 \rightarrow 200]$} \\
      \textbf{RL} & & \\
      Model & \multicolumn{2}{c}{SAC+HER} \\
      Training Timesteps & $2\cdot10^5$ & $2.5\cdot10^5$ \\
      \# AL-STL Cycles & 5 & 5 \\
      Policy Network & \multicolumn{2}{c}{Shared $[64 \rightarrow 64]$} \\
      Learning Rate & \multicolumn{2}{c}{$3\cdot10^-4$} \\
      Discount Factor $\gamma$ & \multicolumn{2}{c}{0.95} \\
      Learning Starts & \multicolumn{2}{c}{100} \\
      Batch Size & \multicolumn{2}{c}{256} \\
      Polyak Update $\tau$ & \multicolumn{2}{c}{0.005} \\
      PGA $\lambda$ &  \multicolumn{2}{c}{0.9} \\
      Training Success Rate & \multicolumn{2}{c}{100\%} \\
      Test Success Rate & \multicolumn{2}{c}{100\%} \\
      \bottomrule
	\end{tabular}
\end{table}

\subsection{Task - Placing Cube}
The hyperparameters are given in \autoref{tab:pandp_params}. The specifications for both these tasks are:
\begin{enumerate}
    \item Placing the cube at the target pose: $\phix{1} \defeq \ev(\norm{cube_{pose}-target_{pose}} \leq 0.05)$.
    \item Reaching the target as quickly as possible: $\phix{2} \defeq \alw(t <= 50)$, where $t$ is the time when the end-effector reaches the target.
\end{enumerate}

The statistics of the PGA shows that is maximum value is $\approx 6$ since there are 2 specifications, each scaled by a factor of 3. 
\begin{figure}[htbp]
    \centering
    \subfloat[Success rates on test trials]{\includegraphics[scale=0.5]{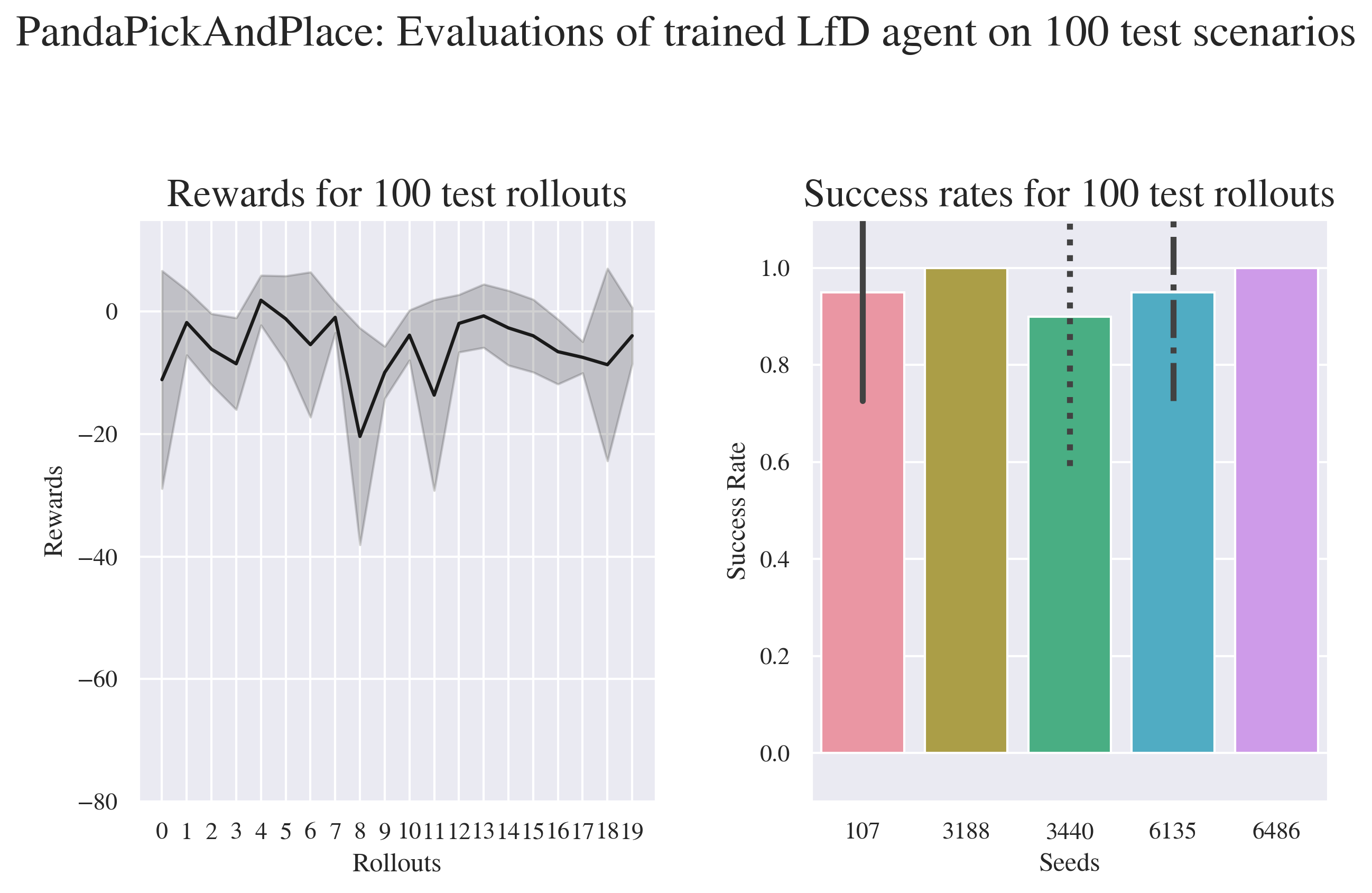}}

    \subfloat[PGA and Weights]{\includegraphics[scale=0.5]{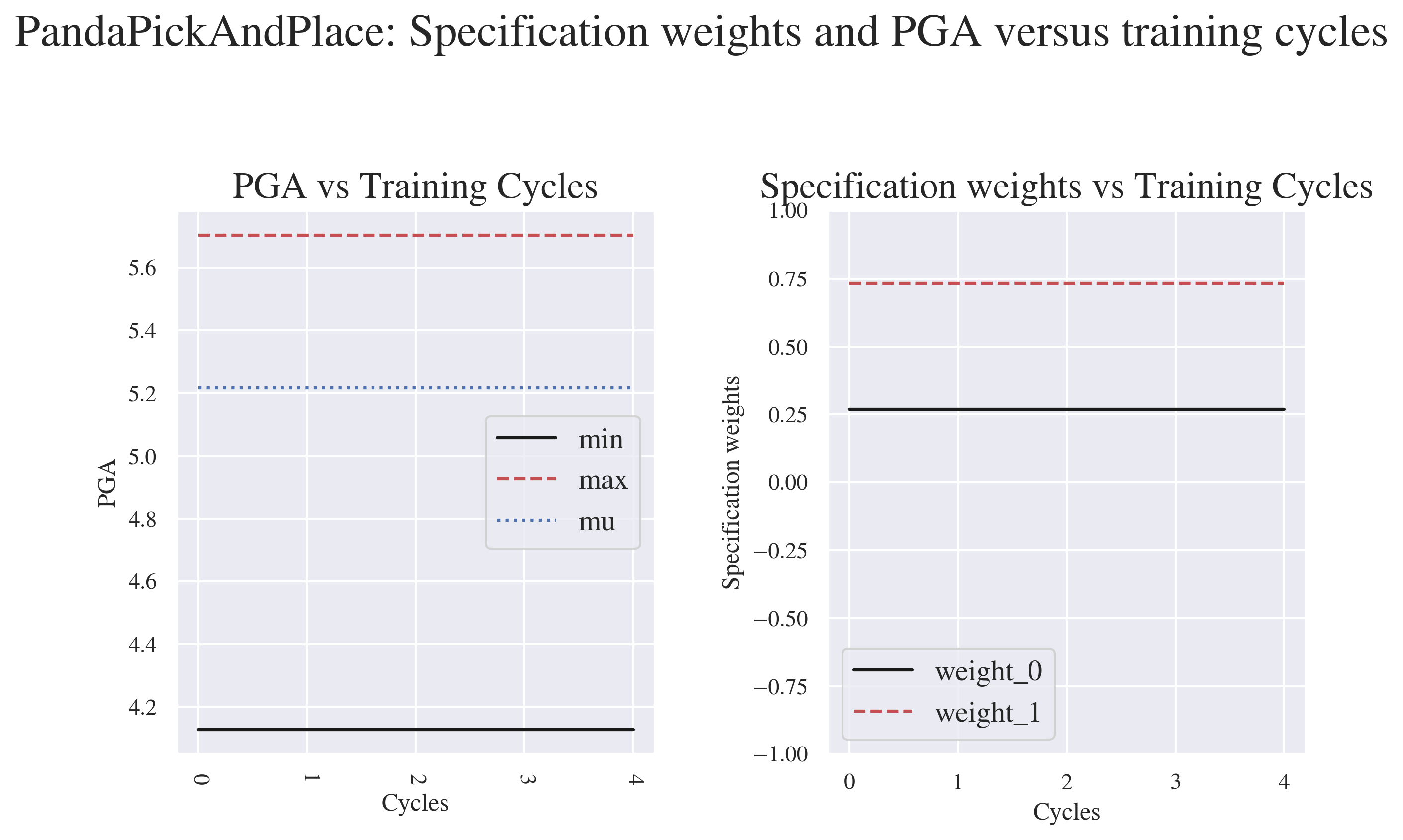}} 

    \subfloat[RL training summary]{\includegraphics[scale=0.5]{object_manipulations/PandaPickAndPlace_SuccessRates}} 
    \caption{Summary of training and evaluations for the Cube-Placing task.}
    \label{fig:app:panda_pandp}
\end{figure}

\begin{table}[htbp]
	\caption{Hyperparameters for cube-placing task.\label{tab:pandp_params}}
	\centering
    \begin{tabular}{c|c}
      \toprule
      Parameters & Value \\
      \midrule
      \# Demos & 5 \\
      Reward Model & Gaussian Process (Scale+RBF kernels) \\
      \textbf{RL} & \\
      Model & TQC+HER \\
      Training Timesteps & $10^7$ \\
      \# AL-STL Cycles & 5 \\
      Policy Network & Shared $[512 \rightarrow 512 \rightarrow 512]$\\
      Learning Rate & $1 \cdot 10^{-3}$ \\
      Discount Factor $\gamma$ & 0.95 \\
      Learning Starts & 1000 \\
      Batch Size & 2048 \\
      Polyak Update $\tau$ & 0.05 \\
      PGA $\lambda$ & 0.9 \\
      Training Success Rate & 98\% \\
      Test Success Rate & 96\% \\
      Training Time & 10.75 hours (2.15 hours/cycle) \\
      \bottomrule
	\end{tabular}
\end{table}

\subsection{Task - Opening Door}
The Panda robot uses operational space control to control the pose of the end-effector. The horizon for this task is 500 and the control frequency is 20 Hz. The hyperparameters are given in \autoref{tab:door_params}. The specifications for both these tasks are:
\begin{enumerate}
    \item Opening the door: $\phix{1} \defeq \ev(\angle door\_hinge \geq 0.3)$. Angle is measured in radians.
    \item Reaching the door handle: $\phix{2} \defeq \ev(\norm{ee - door\_handle} < 0.2)$; end-effector should be within $2cm$ of the door handle.
\end{enumerate}

\begin{figure}[htbp]
    \centering
    \subfloat[Success rates on test trials]{\includegraphics[scale=0.5]{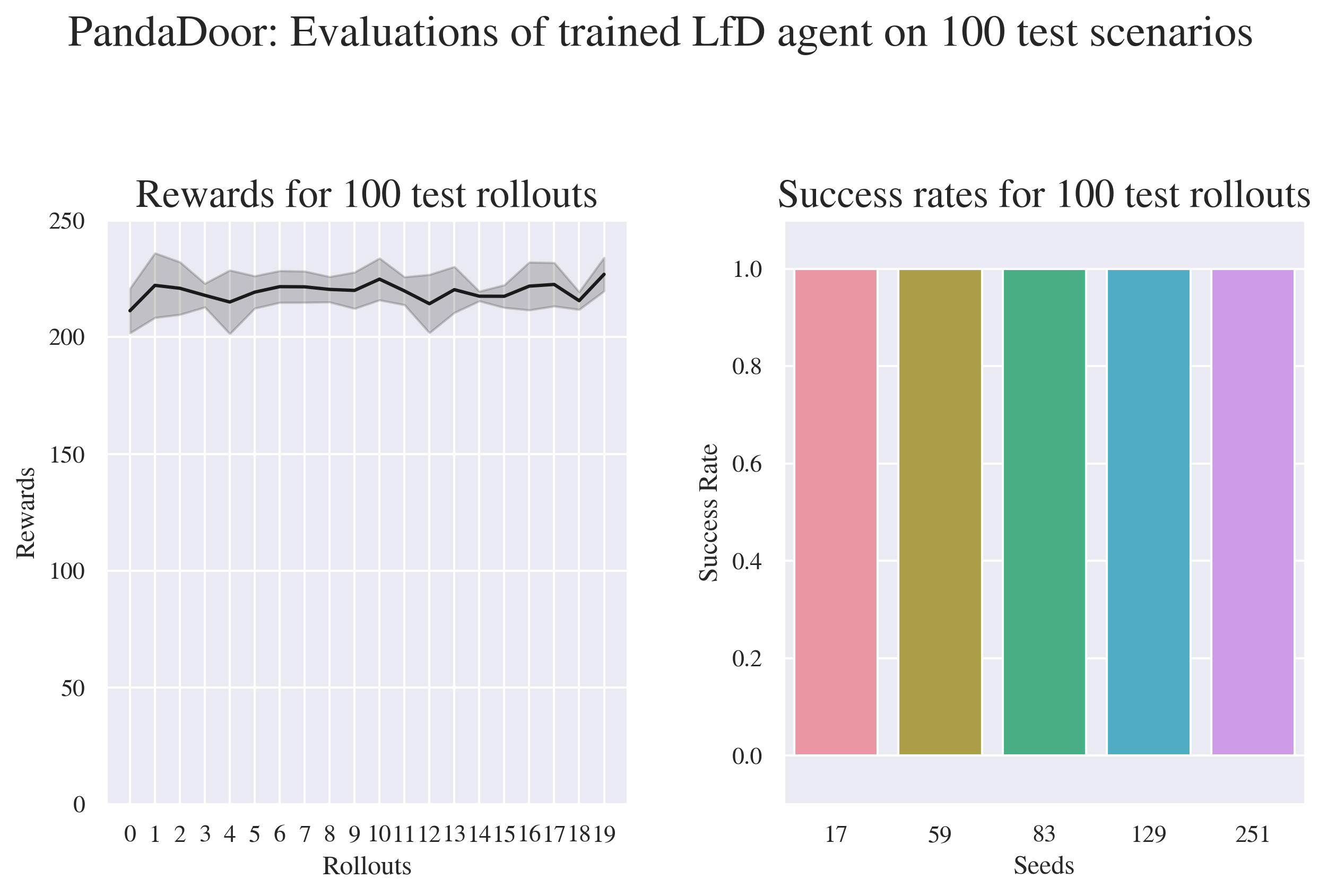}}

    \subfloat[PGA and Weights]{\includegraphics[scale=0.5]{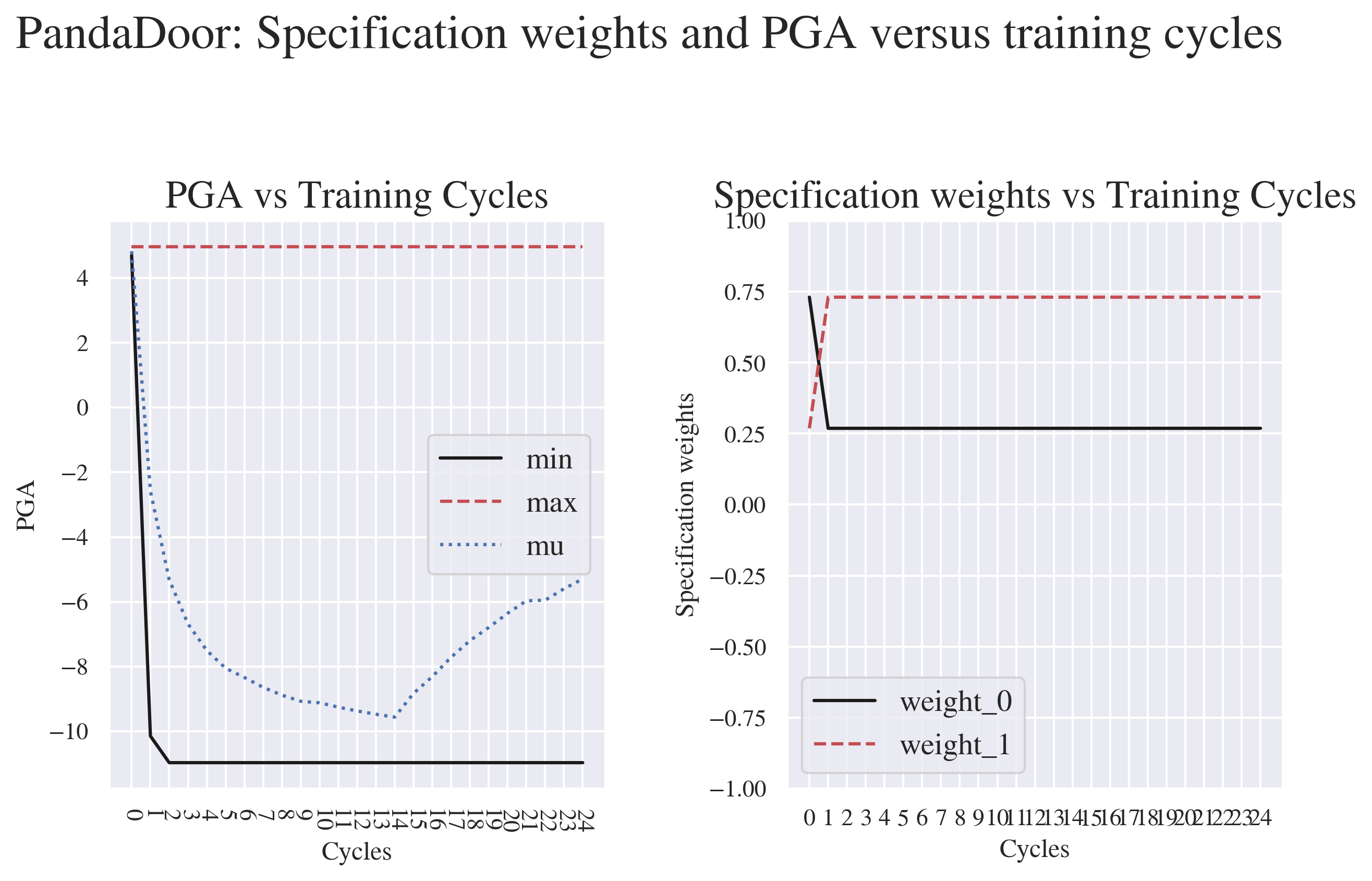}} 

    \subfloat[RL training summary]{\includegraphics[scale=0.5]{object_manipulations/PandaDoor_SuccessRates}} 
    \caption{Summary of training and evaluations for the Door-Opening task.}
    \label{fig:app:panda_door}
\end{figure}

\begin{table}[htbp]
	\caption{Hyperparameters for door-opening task. \label{tab:door_params}}
	\centering
    \begin{tabular}{c|c}
        \toprule
        Parameters & Value \\
        \midrule
        \# Demos & 5 \\
        Reward Model & Neural Network $[16 \rightarrow 16 \rightarrow 16]$ \\
        \textbf{RL} & \\
        Model & TQC \\
        Training Timesteps & $5 \cdot 10^6$ \\
        \# AL-STL Cycles & 25 \\
        Policy Network & Shared $[256 \rightarrow 256]$\\
        Learning Rate & $1 \cdot 10^{-3}$ \\
        Discount Factor $\gamma$ & 0.97 \\
        Learning Starts & 100 \\
        Batch Size & 256 \\
        Polyak Update $\tau$ & 0.5 \\
        PGA $\lambda$ & 0.3 \\
        Training Success Rate & 98\% \\
        Test Success Rate & 100\% \\
        Training Time & 6.5 hours (0.26 hours/cycle) \\
        \bottomrule
      \end{tabular}
\end{table}

\end{document}